\pdfoutput=1
\documentclass[a4paper]{styles/svproc}
 \usepackage{url}

 \usepackage[T1]{fontenc}
 \usepackage{graphics} 
 \usepackage{epsfig} 
 \usepackage{times} 
 \usepackage{subfigure}
 \usepackage{amssymb}  
 \usepackage{mathtools}
 \usepackage{pkgs/commath}

 \usepackage{amsthm}

 \usepackage{array}
 \newcolumntype{M}[1]{>{\centering\arraybackslash}m{#1}}
 
 \usepackage{array}
 \usepackage{booktabs}
 \newcommand{\PreserveBackslash}[1]{\let\temp=\\#1\let\\=\temp}
 \newcolumntype{C}[1]{>{\PreserveBackslash\centering}m{#1}}
 \newcolumntype{R}[1]{>{\PreserveBackslash\raggedleft}p{#1}}
 \newcolumntype{L}[1]{>{\PreserveBackslash\raggedright}p{#1}}
 
 \usepackage{colortbl}
 \usepackage{tabulary}

 \newtheorem{thm}{Theorem}

 \newtheorem{lem}[thm]{Lemma}
 \usepackage{amsmath,amsfonts} 

 \usepackage{amsmath,amssymb,amsfonts}
 \usepackage{graphicx}
 \DeclareGraphicsExtensions{.eps,.pdf,.png,.jpg,.JPG,.gif}
 \usepackage{subfloat}
 \usepackage{pkgs/algorithm}
 \usepackage{pkgs/algorithmic}
 \usepackage{bm}
 \usepackage[backend=bibtex,natbib=true,style=numeric-comp,sorting=none,giveninits=true,maxbibnames=99,url=false,doi=false]{biblatex}
\begin{document}
\mainmatter              
%
\title{Visual-Inertial Localization for Skid-Steering Robots with Kinematic Constraints}
%
\titlerunning{Kinematics-constrained visual-inertial localization}  
%
\author{Xingxing Zuo\inst{1}$^,$\inst{2}$^\ddag$  \and Mingming Zhang\inst{2}$^\ddag$  \and Yiming Chen\inst{2}  \and Yong liu\inst{1} \and \\ Guoquan Huang\inst{3} \and Mingyang Li\inst{2} 
\thanks{$^\ddag$ X. Zuo and M. Zhang contribute equally to this work.
This work was supported by supported by Alibaba-Zhejiang University Joint Institute of Frontier Technologies.
} }
\authorrunning{Xingxing Zuo et al.} 
%
\tocauthor{Xingxing Zuo， Mingyang Li}
\institute{Institute of Cyber-System and Control, Zhejiang University, China,\\
\email{ xingxingzuo@zju.edu.cn,yongliu@iipc.zju.edu.cn}
\and
 A.I. Labs, Alibaba Group, China, \\
\email{
$\{$mingming.zhang, yiming.chen,
mingyangli$\}$@alibaba-inc.com}
\and
Department of Mechanical Engineering, University of Delaware, USA \\
\email{ghuang@udel.edu}
}
\maketitle              

\begin{abstract}

While visual localization or SLAM has witnessed great progress in past decades, when deploying it on a mobile robot in practice, 
few works have explicitly considered the kinematic (or dynamic) constraints of the real robotic system when designing state estimators.
To promote the practical deployment of current state-of-the-art visual-inertial localization algorithms, 
in this work we propose a low-cost kinematics-constrained localization system particularly for a skid-steering mobile robot.
In particular, we derive in a principle way the robot's kinematic constraints based on the instantaneous centers of rotation (ICR) model 
and integrate them in a tightly-coupled manner into the sliding-window bundle adjustment (BA)-based visual-inertial estimator.
Because the ICR model parameters are time-varying due to, for example, track-to-terrain interaction and terrain roughness,
we estimate these kinematic parameters online along with the navigation state. 
To this end, we perform in-depth the observability analysis and identify motion conditions under which the state/parameter estimation is viable.
The proposed kinematics-constrained visual-inertial localization system has been validated extensively in different terrain scenarios.

\end{abstract}

\section{Introduction}
\label{sec:intro}

It is essential for mobile robots to perform robust and accurate real-time localization when deployed in real-world applications such as autonomous delivery.
While visual localization or SLAM has made significant progress in the last decades, most of current algorithms are generally purposed and not tailored to particular robotic systems --
that is, their design is often independent of robots. 
However, the robotic system  can provide informative state constraints due to its dynamics and/or kinematics, which should be exploited when designing localization algorithms for robots at hand.
In this paper, bearing this in our mind, we develop a kinematics-constrained visual-inertial localization algorithm for skid-steering robots,
which tightly fuses low-cost camera, IMU and odometer sensors to provide high-precision real-time localization solutions in 3D.

\begin{figure}[t]
	\centering
	\subfigure[]{ 
		\includegraphics[width=2.0in]{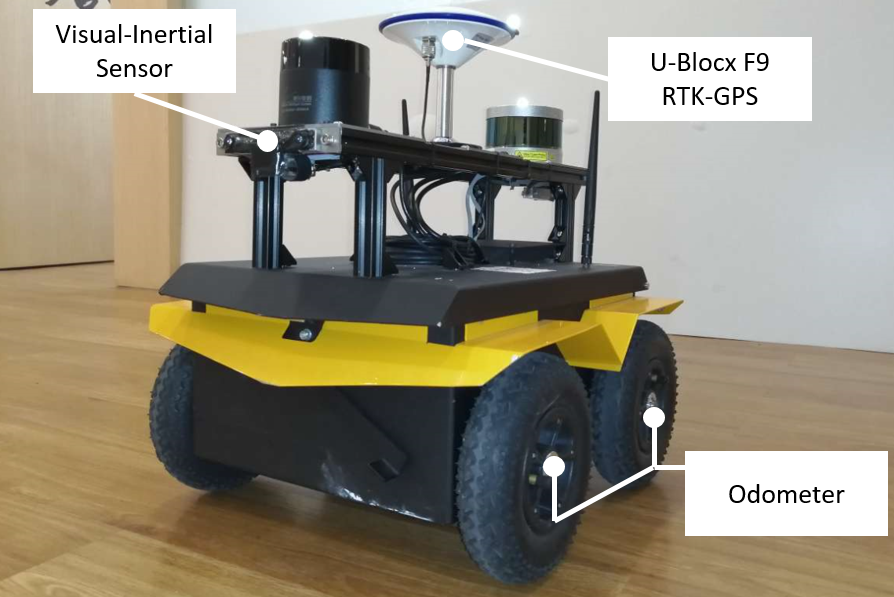} 
	} 
	\subfigure[]{ 
		\includegraphics[width=2.2in]{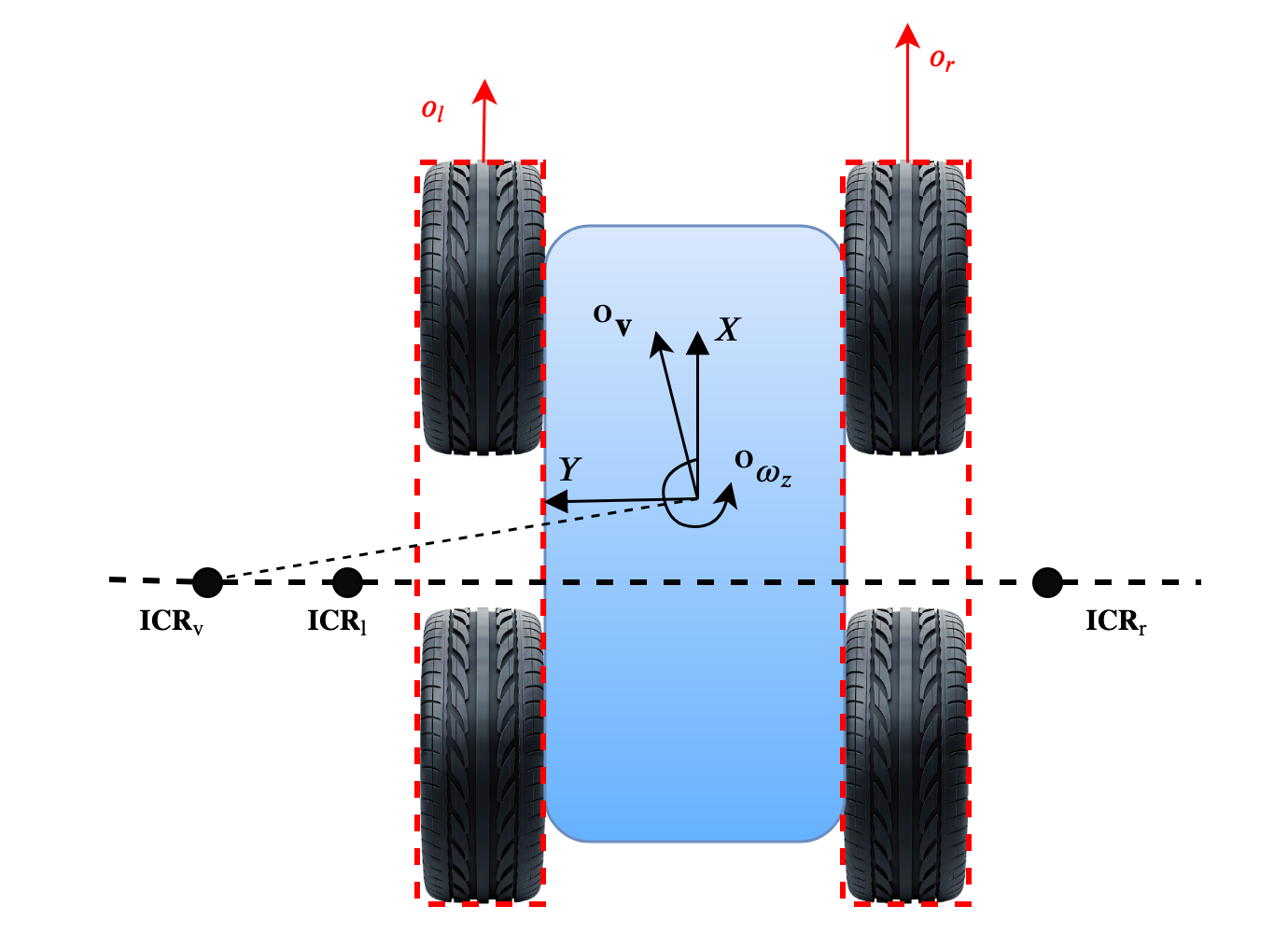}
	} 
	\caption{The skid-steering robotic platform used in our test and the ICR based kinematics model. (a) Clearpath Jackal Platform~\cite{Clearpath}. The equipped sensors include the ones for performing real-time localization (i.e., camera, imu, and wheel odometers), and the ones for providing ground truth poses in experiments (i.e.,RTK-GPS). (b) The odometer measurements and the instantaneous center of rotation (ICR) of a skid-steering robot. $\mathrm{\mathbf{ICR}}_v, \mathrm{\mathbf{ICR}}_l, \mathrm{\mathbf{ICR}}_r$ denote the ICR positions of the robot frame, left wheels and right wheels, respectively. ${}^\mathbf{O}\mathbf{v}$ represents the robot velocity in odometer frame, and ${}^\mathbf{O}\omega_z$ is the angular velocity in the yaw direction.}
	\label{fig:robots_icr}
\end{figure}

Visual-inertial sensors are becoming ubiquitous and many general-purpose visual-inertial navigation algorithms have been developed in recent years (e.g., see~\cite{Li2013high,li2014high,qin2018vins,xing2017photometric,Eckenhoff2019IJRR}),
which has motivated an increasing number of deployments of such sensor suite on real robotic systems~\cite{Hartley2018IROS}.
Due to their low cost and complementary sensing capabilities, we have also employed them in our proposed skid-steering robotic system (see Fig.~\ref{fig:robots_icr}).
Note that, instead of having explicit mechanism of steering control, skid-steering robots rely on adjusting the speed of left and right tracks to turn around. 
The simplicity of the mechanical design and the ability to turn around with zero-radius have made such robots popular in scientific research and development.
%
%

Due to the popularity of skid-steering robots,
substantial research efforts have focused on the motion dynamics modeling, control, and planning~\cite{martinez2005approximating,huskic2017path,pentzer2014use}. 
In particular, \citet{yi2009kinematic} introduced a simple dead-reckoning (DR) method for skid-steering robots,
while \citet{wang2018terrain} relied on accurate GPS to provide localization (which clearly is not applicable if GPS is not available or reliable).
As the closest to this work,  \citet{wu2017vins} recently proposed a visual-inertial localization method for wheeled vehicles by directly using an odometer's 2D linear/angular velocity measurements.
%
While this approach~\cite{wu2017vins} is perfectly suitable for a standard differential-drive robot, 
significant efforts on kinematic modeling and fusion may be required to deploy it on skid-steering robots;
if blindly ignoring that, localization performance would be  degraded.

To address these issues and promote visual-inertial localization for skid-steering robots, 
in this paper, we, for the first time, design a tightly-coupled visual-inertial estimation algorithm that fully exploits the robot's ICR-based kinematic~\cite{martinez2005approximating} constraints and efficiently offers 3D localization solutions.
In particular, to compensate for the time-varying ICR model parameters (e.g., due to slippage and terrain roughness),
we explicitly model and estimate online the kinematic parameters of a skid-steering robot.
To this end, leveraging our significant prior work on visual-inertial odometry~\cite{li2014online,Li2013high},
we develop an efficient sliding-window bundle adjustment (BA)-based estimator to optimally fuse measurements from a camera, an IMU, and wheel encoders. 
%
%
Moreover, we have performed observability analysis in detail, showing that the kinematic parameters are all {observable} under general motions
while the observability would not hold when the IMU is not used,
which is important for estimator design.


\section{Related Work}
\label{sec:relatedwork}




As there is rich literature on mobile robot localization~\cite{Cadena2016TRO},
by no means, we intend to provide a comprehensive review on this topic and instead focus on wheeled robots here.
%
For example, 
\citet{censi2013simultaneous} performed pose estimation with online wheel odometry parameter (the radius of left and right wheels as well as the distance between them) calibration for  a differential drive robot  equipped with two wheels,
while \citet{scaramuzza20111} introduced a camera based localization algorithm for   Ackermann model-based wheeled robots.
As mentioned earlier, 
\citet{wu2017vins} developed a sliding-window EKF to probabilistically fuse the measurements from wheel encoders, an IMU, and a monocular camera to provide 6DOF motion. 
\citet{yap2011particle} solved the similar problem but with a particle filter based method.
%
However, in all of these methods, it is assumed that linear and angular velocities of a robot can be directly computed from wheel encoder readings, 
which is {\em not} the case for skid-steering robots.


A skid-steering robot often uses the ICR positions of treads to model its motion dynamics~\cite{martinez2005approximating}.
Since it was found empirically that the ICR parameters have small variations under same terrain conditions~\cite{martinez2005approximating},
additional modeling parameters were introduced for better modelling. 
For example, \citet{huskic2017path} used additional scale variables  for allowing accurate path following
and \citet{martinez2017inertia} modeled additional  sliding, eccentricity and steering efficiency.
Note that ICR is not the only model for skid-steering robots and there are many others.
For instance, \citet{reina2016slip} modeled  the distance between left and right tread  and integrated it into terrain classification.
\citet{sutoh2018motion}
modeled the ratio of the velocities between left and right wheels as an exponential function of ratio of readings between left and right wheel encoders,
and these exponential parameters are estimated during terrain navigation.

Depending on sensors used and application scenarios, 
different localization algorithms for skid-steering robots have been developed in recent literature.
In particular,
\citet{lv2019fvo} proposed a method for using images for correcting headings for skid-steering robot, 
while 
requiring parallel and perpendicular lines which mainly are suitable for human-made environments.
%
which however did not provide detailed description of how the wheel encoder's measurements were integrated.
%
IMU measurements are typically used together with wheel encoder readings to provide motion tracking of skid-sterring robots.
For example, \citet{yi2009kinematic} used an IMU on the skid-steering robot to perform both trajectory tracking and slippery estimation,
and \citet{lv2017indoor} fused measurements from wheel encoders, a gyroscope, and a magnetometer to localize the skid-steering robot. 
GPS measurements, if available, are also leveraged with EKF~\cite{pentzer2014model},
in which the ICR locations were modeled as parts of the state vector and estimated online. 
Specifically, the wheel encoder measurements were used for EKF pose prediction and the GPS measurements were used for EKF update. 
\citet{wang2018terrain} combined both GPS and IMU measurements, in which they used GPS to perform high-precision navigation and rely on accelerometer measurements 
terrain classification.
In contrast, in this paper, we focus on skid-steering robot localization with low-cost multi-modal sensors while integrating kinematic constraints.



\section{ICR-based Kinematics of Skid-Steering Robots}
\label{sec:icr-kinmatics}

In this work, we employ the ICR parameters~\cite{martinez2005approximating} to approximately model the kinematics of a skid-steering robot. 
Specifically, as shown in Fig.~\ref{fig:robots_icr}, 
we denote $\mathrm{\mathbf{ICR}}_v = \left( X_v, Y_v \right)$  the ICR position of the robot frame, and 
$\mathrm{\mathbf{ICR}}_l = \left( X_l, Y_l \right)$ and $\mathrm{\mathbf{ICR}}_r = \left( X_r, Y_r \right)$  the ones of the left and right wheels, respectively.
The relation between  the readings of wheel odometer measurements and the ICR parameters can be derived as follows:\footnote{Throughout this paper, 
the robot is equipped with a camera, an IMU, and wheel odometers, whose frames are denoted by $\lbrace \mathbf{C} \rbrace$, $\lbrace \mathbf{I} \rbrace$, and $\lbrace \mathbf{O}\rbrace$, respectively,
while $\lbrace \mathbf{G} \rbrace$ refers to the global frame of reference.
 $^{\mathbf A}\mathbf p_{\mathbf B}$ and $^{\mathbf A}_{\mathbf B} \mathbf R$  denote the 3DOF position and rotation of frame $\lbrace\mathbf{B}\rbrace$ with respect to $\lbrace\mathbf{A}\rbrace$.
We use $\hat{\mathbf x}$ and $\delta \mathbf x$ to represent the estimate of random variable $\mathbf x$ and  its error state.
The symbol $\breve{z}$ is used to denote the inferred measurement mean value of $z$.
}
\begin{align}
Y_l & =  - \frac{ o_l  - {}^{\mathbf  O}v_{x}}{{}^{\mathbf  O}\omega_z},\,\,\,Y_r  =  - \frac{ o_r  - {}^{\mathbf  O}v_{x}}{{}^{\mathbf  O}\omega_z} \notag \\
\label{eq:ICR}
Y_v & =  \frac{{}^{\mathbf  O}v_{x}}{{}^{\mathbf  O}\omega_z},\,\,\,
X_v  = X_l = X_r=  - \frac{{}^{\mathbf  O}v_{y}}{{}^{\mathbf  O}\omega_z}
\end{align}
%
%
where $o_l$ and $o_r$ are linear velocities of left and right wheels, ${}^{\mathbf  O}v_{x}$ and ${}^{\mathbf  O}v_{y}$ are robot's local linear velocity along $x$ and $y$ axes defined in Fig.~\ref{fig:robots_icr}, and ${}^{\mathbf  O}\omega_z$ denotes the local rotational speed.
%
%
Moreover, we introduce two additional scale factors,  $\left[\alpha_l, \alpha_r\right]$, to compensate for the possible effects, e.g., 
due to tire inflation and interface roughness. 
With the scale factors and Eq.~\ref{eq:ICR}, we can express the motion variables as:
\begin{align}~\label{eq:ICRmat_5par}
\begin{bmatrix} {}^{\mathbf  O}v_{x} \\ {}^{\mathbf  O}v_{y} \\ {}^{\mathbf  O}\omega_z \end{bmatrix} =
g(\boldsymbol \xi, o_l, o_r) = 
\frac{1}{\Delta Y} 
\begin{bmatrix}
- Y_r &  Y_l\\ 
X_v  & - X_v\\
-1 & 1
\end{bmatrix}
\begin{bmatrix}
\alpha_l &0 \\ 0 &\alpha_r
\end{bmatrix}
\begin{bmatrix}
o_l \\ o_r
\end{bmatrix},\,\,\,\,
\boldsymbol \xi = 
\begin{bmatrix}
X_v \\ 
Y_l \\ Y_r \\ 
\alpha_l \\
\alpha_r
\end{bmatrix}
\end{align}
where ${\Delta Y}  =  Y_l -  Y_r$, and $\boldsymbol \xi$ is the entire set of kinematic parameters.
%
%

Interestingly, as a special configuration when $\boldsymbol \xi = [0, \frac{b}{2}, \frac{-b}{2}, 1, 1]^T$,
with $b$ being the distance between left and right wheels, Eq.~\ref{eq:ICRmat_5par} can be simplified as:
\begin{align}
\label{eq:ordinarymodel}
	{}^{\mathbf  O}v_{x} = \frac{o_l + o_r}{2}, ~~{}^{\mathbf  O}\omega_z = \frac{o_r - o_l}{b},~~ {}^{\mathbf  O}v_{y} =0
\end{align}
This is the kinematic model for a wheeled robot moving without slippage (e.g., a differential drive robot), 
and used by most existing work for localizing wheeled robots~\cite{wu2017vins,quan2018tightly}. 
However, in the case of skid-steering robots under consideration, if  directly applying Eq.~\ref{eq:ordinarymodel},
the localization accuracy would be significantly degraded (see Section~\ref{sec:exp}).
%
%
It is important to point out that
as $\boldsymbol \xi$ cannot remain constant due to different motions and terrains~\cite{martinez2005approximating,huskic2017path},
we will perform online ``calibration'' to estimate these kinematic parameters along with the navigation states as in~\cite{li2014online,censi2013simultaneous,Li2013high} (see Section~\ref{sec:icr_pred}).  

\section{Kinematics-Constrained Visual-Inertial Localization}
\label{sec:estimator}

We develop a window-BA estimator for the proposed kinematics-constrained visual-inertial localization for a skid-steering robot equipped with a camera, an IMU, and wheel encoders.
For simplicity, although not necessary, we assume known extrinsic transformations between sensors.
%
At each time step, we optimize the following window of states, 
whose typically oldest state will be marginalized out when moving to the next window in order to bound computational cost:
\begin{align}
	\mathbf{x} = \lbrace {}^{\mathbf{G}}_{\mathbf{O}}\mathcal{T}, {}^{\mathbf{G}}\mathbf{v}_{\mathbf{I}_k}, \mathbf{b}_a, \mathbf{b}_\omega, \boldsymbol{\xi}, \mathcal{F}, \mathbf{m} \rbrace
\end{align}
In the above expression, 
${}^{\mathbf{G}}_{\mathbf{O}}\mathcal{T} = \lbrace {}^{\mathbf{G}}_{\mathbf{O}_{k-s}}\mathbf{T}, \dots, {}^{\mathbf{G}}_{\mathbf{O}_{k-1}}\mathbf{T}, {}^{\mathbf{G}}_{\mathbf{O}_{k}}\mathbf{T}\rbrace$ denotes the cloned poses in the sliding window at time $\lbrace k-s, \dots, k\rbrace$. 
${}^{\mathbf{G}}_{\mathbf{O}_{k}}\mathbf{T} = \left\{ {}^{\mathbf{G}}_{\mathbf{O}_{k}}\mathbf{R}, {}^{\mathbf{G}}\mathbf{p}_{\mathbf{O}_k} \right\}$ represents the 6DOF pose of the robot at time $k$. 
We choose the odometry frame is the base sensor frame and the system is initialized by the initial position of odometer while the direction of $z$ is aligned with the gravity. 
$\mathcal{F}$ contains all the 3D global positions of visual features. 
$ {}^{\mathbf{G}}\mathbf{v}_{\mathbf{I}_k}, \mathbf{b}_a, \mathbf{b}_\omega$ are the IMU velocity in global frame, acceleration bias and angular velocity bias, respectively. 
Note that we estimate online the ICR kinematic parameters $\boldsymbol{\xi}$ and thus include them in the state as well.
Lastly, $\mathbf{m}$ denotes the parameters related to the motion manifold constraints enforcing local smooth ground planar  motion.
As illustrated in Fig.~\ref{fig:factorgraph}, the sliding window BA is our estimation engine whose cost function includes the following constraints:
\begin{align} \label{eq:cost}
	\mathcal{C} = \mathcal{C}_{prior} + \mathcal{C}_{proj} + \mathcal{C}_{I} + \mathcal{C}_{odom} +\mathcal{C}_{manifold} 
\end{align}
which includes 
the prior  of the states remaining in the current sliding window after marginalization~\cite{Eckenhoff2019IJRR},
the projection error of visual features, 
the IMU integration constraints~\cite{Eckenhoff2019IJRR,Li2013high}, 
the odometer-induced kinematic constraints, 
and the motion manifold constraints.
%

%
%

 \begin{figure} 
	\centering
	\includegraphics[width = 0.8\textwidth]{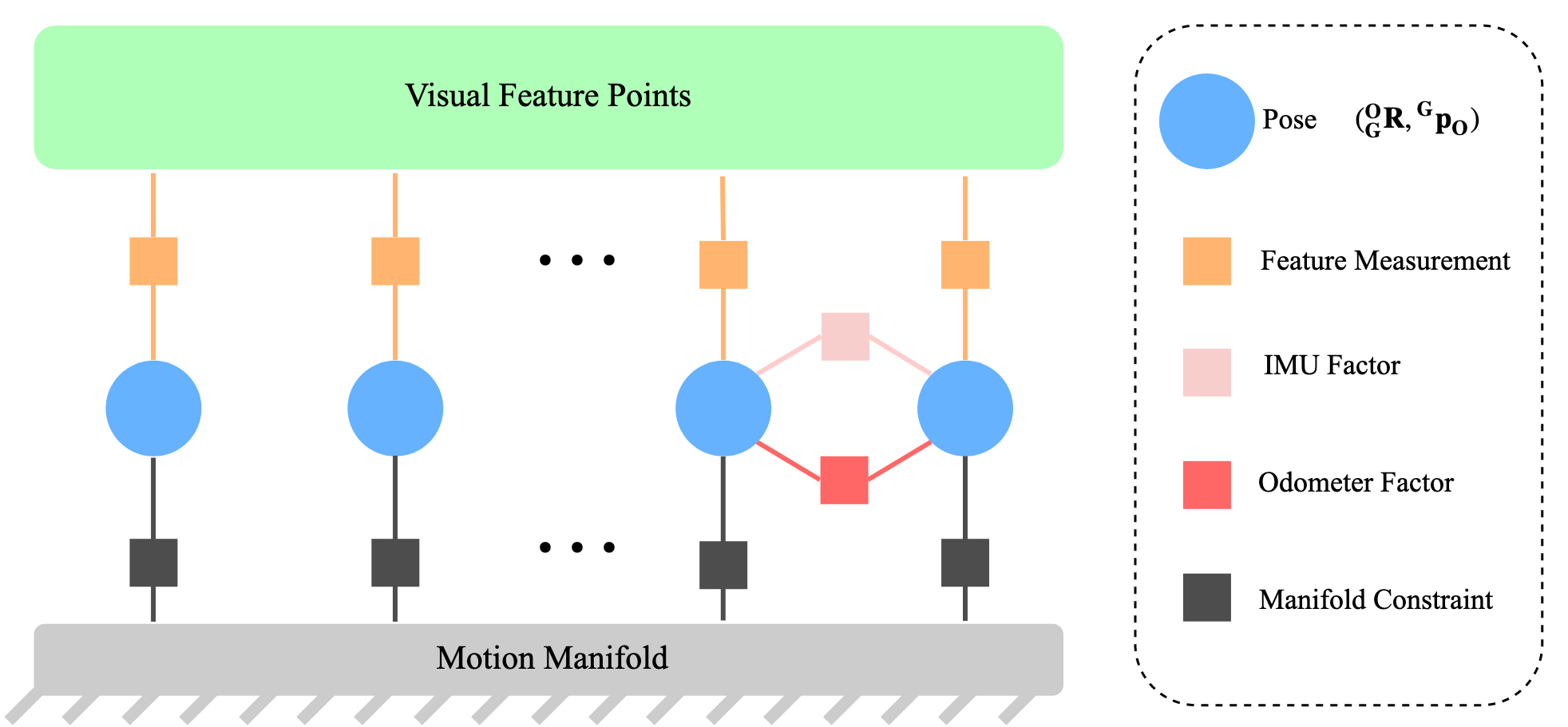}
	\caption{In the proposed kinematics-constrained visual-inertial localization for skid-steering robots, five different constraints are used in the sliding-window BA: A prior encapsulates the information about the current states due to marginalization of states and measurements (Prior factors are related to all states related to marginalized measurements, we omit to plot them in this figure for clarity);
	Visual feature measurements connect the feature points in the map and the robot pose at the time when the image was recorded; 
	IMU preintegration factors summarize the sequential IMU raw measurements between the two images; 
	Odometr-induced kinematic factor summaries the sequential odometer measurements between the two images. 
	} 
	\label{fig:factorgraph} 
\end{figure}

\subsection{Visual-Inertial Constraints}
\label{sec:vi-constraint}

In the sliding-window BA, only keyframes are optimized, which are selected based on a simple heuristic: the odometer prediction has a translation or rotation over a certain threshold (e.g., 0.2 meter and 3 degrees as in our experiments).
In contrast, for computational savings, non-keyframes will be discarded,
unlike existing methods~\cite{qin2018vins,leutenegger2015keyframe} which extract features firstly and analyses the distribution of the features for keyframe selection. 
%
%
Among keyframes in the window, 
corner feature points are extracted~\cite{rosten2006machine}   and tacked by KLT optical flow algorithm~\cite{lucas1981iterative}.  
The standard reprojection errors of the tracked features comprise the visual cost $\mathcal{C}_{proj}$ in \eqref{eq:cost} as in \cite{Eckenhoff2019IJRR,Li2013high}.
%

On the other hand, 
the IMU measurements between any two consecutive keyframes are integrated and form the inertial constraints across the sliding window~\cite{Eckenhoff2019IJRR,Li2013high}:
\begin{align}
	\mathcal{C}_{I} = \big{|}\big{|} {\mathbf{x}}_{I_k} \boxminus f \left(  {\mathbf{x}}_{I_{k-1}}, \mathcal{I}_{am}, \mathcal{I}_{\omega m} \right) \big{|}\big{|}_{\bm \Lambda_{I_k}}^2
\end{align}
%
where 
${\mathbf{x}}_{Ik}$ 
is the IMU state at time $t_{k}$, and $\mathcal{I}_{am}, \mathcal{I}_{\omega m}$ denote the IMU acceleration and angular velocity measurements between $t_{k-1}$ and $t_k$, respectively. ${\bm \Lambda_{I_k}}$ represents the inverse covariance (information) of the IMU prediction $f \left(  {\mathbf{x}}_{I_{k-1}}, \mathcal{I}_{am}, \mathcal{I}_{\omega m} \right)$.

\subsection{ICR-based Kinematic Constraints}
\label{sec:icr_pred}

We now derive the ICR-based kinematic constraints based on the wheel encoders' measurements of the skid-steering robot.
Specifically, 
by assuming the supporting manifold of the robot is locally planar between $t_k$ and $t_{k+1}$, 
the local linear and angular velocities, ${}^{\mathbf O{(t)}}{\mathbf{v}}$ and ${^{\mathbf O(t)} \boldsymbol \omega}$, 
are a function of the wheel encoders' measurements of the left and right wheels $o_{lm}(t)$ and $o_{rm}(t)$ 
as well as the ICR kinematic parameters $\boldsymbol \xi$ [see \eqref{eq:ICRmat_5par}]:
\begin{align}
 \left[ {}^{\mathbf O_{(t)}}{\mathbf{v}}^T , {^{\mathbf O(t)} \boldsymbol \omega}^T \right]^\top &= \boldsymbol{\Pi} \, g(\boldsymbol{\xi}(t), o_l(t), o_r(t) ) \notag \\
 &= \boldsymbol{\Pi} \,
 g(\boldsymbol{\xi}(t), o_{lm}(t) - n_l(t), o_{rm}(t) - n_r(t) )
\end{align}
where $\boldsymbol{\Pi} = 
    \begin{bmatrix}
    \mathbf e_1^T &
    \mathbf e_2^T &
    \mathbf 0 &
    \mathbf 0 &
    \mathbf 0 &
    \mathbf e_3^T
    \end{bmatrix}^T$ is the selection matrix 
    with    $\mathbf e_i$ being  a $3\times1$ unit vector with the $i$th element of 1,
$n_l(t)$ and $n_r(t)$ are the odometry noise  modeled as zero-mean white Gaussian.
Once the instantaneous local velocities of the robot are available, 
with the initial conditions ${}^{\mathbf O_{k-1}}_{\mathbf O(t)}\mathbf{R} \big{|}_{t = t_{k-1}} = \mathbf I_{3\times3}$ and ${}^{\mathbf O_{k-1}}{\mathbf{p}}_{\mathbf O(t)} \big{|}_{t = t_{k-1}} = \mathbf 0_{3\times1}$,
we can integrate the following differential equations in the time interval $t\in [t_{k-1}, t_{k}]$:
\begin{align}
    {}^{\mathbf O_{k-1}}_{\mathbf O(t)}\dot{\mathbf{R}} &=
    {}^{\mathbf O_{k-1}}_{\mathbf O(t)}\mathbf{R} \cdot 
    \lfloor {^{\mathbf O(t)} \boldsymbol \omega} \rfloor \notag \\
    	\label{eq:propagated_raw2}
    {}^{\mathbf O_{k-1}}\dot{\mathbf{p}}_{\mathbf O(t)} &= 
    {}^{\mathbf O_{k-1}}{\mathbf{v}}_{\mathbf O(t)} = 
    {}^{\mathbf O_{k-1}}_{\mathbf O(t)}\mathbf{R} \cdot
    {}^{\mathbf O_{(t)}}{\mathbf{v}} 
\end{align}
This integration will result in the relative pose $\{{}^{\mathbf O_{k-1}}\mathbf{p}_{\mathbf O_{k}}, {}^{\mathbf O_{k-1}}_{\mathbf O_{k}}\mathbf{R}\}$,
which is then used to propagate the global pose from $t_{k-1}$ to $t_{k}$:
\begin{subequations}
\label{eq:propagated_raw}
\begin{align}
  	{}^{\mathbf G}\mathbf{p}_{\mathbf O_{k}} &= {}^{\mathbf G}\mathbf{p}_{\mathbf O_{k-1}} + {}^{\mathbf G}_{\mathbf O_{k-1}}\mathbf{R} \cdot {}^{\mathbf O_{k-1}}\mathbf{p}_{\mathbf O_{k}} \\
  	{}^{\mathbf G}_{\mathbf O_{k}}\mathbf{R} &= {}^{\mathbf G}_{\mathbf O_{k-1}}\mathbf{R} \cdot {}^{\mathbf O_{k-1}}_{\mathbf O_{k}}\mathbf{R}
\end{align}
\end{subequations}
%
%
%
%
%
Additionally,
we model the ICR kinematic parameter $\boldsymbol \xi$ as a random walk to capture its time-varying characteristics:
\begin{align}
\label{eq:ICR Ki}
    \dot{\boldsymbol{\xi}}(t) = \mathbf n_{\boldsymbol \xi}(t)
\end{align}
where $\mathbf n_{\boldsymbol \xi}$ is zero-mean white Gaussian noise.

Based on the ICR-based kinematic model~\eqref{eq:propagated_raw} and \eqref{eq:ICR Ki}, 
we predict the pose and kinematic parameter at the newest keyframe time $t_k$,
$\hat{\mathbf{x}}_{odom_{k}} = \left[ {}^{\mathbf G}_{\mathbf{O}_k}\hat{ \mathbf{R}}, {}^{\mathbf{G}}\hat{\mathbf{p}}_{\mathbf{O}_k},  \hat{ \boldsymbol{\xi} } \right] = f \left(  {\mathbf{x}}_{{odom}_{k-1}}, \mathcal{O}_{lm}, \mathcal{O}_{rm} \right) $, 
by integrating all the intermediate odometery measurements $\mathcal{O}_{lm}, \mathcal{O}_{rm}$.
As a result, the odometer-induced kinematic constraint can be generically written in the following form:
\begin{align}
	\mathcal{C}_{odom} = \big{|}\big{|} {\mathbf{x}}_{{odom}_k} \boxminus f \left(  {\mathbf{x}}_{{odom}_{k-1}},\mathcal{O}_{lm}, \mathcal{O}_{rm} \right) \big{|}\big{|}_{\bm \Lambda_{odom_k}}^2
\end{align}
where ${\bm \Lambda_{odom_k}} $ represents the inverse covariance (information)  obtained via  covariance propagation.
Specifically, the discrete-time linearized kinematic model of the error state at $t_i$,
$\delta \mathbf x_{i} = [{\delta \mathbf p_{i}}, {\delta \boldsymbol \theta_{i}}, {\delta \boldsymbol \xi}_{i}]$,
corresponding to \eqref{eq:propagated_raw} and \eqref{eq:ICR Ki} at time $t_i$ can be found as follows: 
\begin{align}
    \delta \mathbf x_{i} =
    \boldsymbol \Phi_{i-1}
    \delta \mathbf x_{i-1} + \mathbf{G}_{i-1}
    \mathbf n_{i-1}
\end{align}
%
where $\mathbf n_{i-1} = \left[n_l(t_{i-1}), n_r(t_{i-1}), n_{\boldsymbol \xi}(t_{i-1})\right]^\top \sim \mathcal{N}(\mathbf 0, \mathbf Q_{i-1})$,
and $\boldsymbol \Phi_{i-1}$ is the error-state transition  matrix which is given by:
\begin{align}
	\bm{\Phi}_{i-1} = \begin{bmatrix}
	\mathbf{I}_{3 \times 3} & -{}^{\mathbf G}_{\mathbf O_{i-1}}\hat{\mathbf{R}} \lfloor {}^{\mathbf O_{i-1}}\hat{\mathbf{p}}_{\mathbf O_{i}}\rfloor & \bm{\Phi}_A  &  \bm{\Phi}_B \\ 
	\mathbf{0}_{3 \times 3}& {}^{\mathbf O_{i-1}}_{\mathbf O_{i}}\hat{\mathbf{R}}^\top &\bm{\Phi}_C &\bm{\Phi}_D\\
	\mathbf{0}_{3 \times 3}& 	\mathbf{0}_{3 \times 3}& 	\mathbf{I}_{3 \times 3}& 	\mathbf{0}_{3 \times 2}	\\
	\mathbf{0}_{2 \times 3}& 	\mathbf{0}_{2 \times 3}& 	\mathbf{0}_{2 \times 3}& 	\mathbf{I}_{2 \times 2}	\\
	\end{bmatrix}
\end{align}
where $\boldsymbol{\Phi}_A, \boldsymbol{\Phi}_B, \boldsymbol{\Phi}_C, \boldsymbol{\Phi}_D$ are non-zero blocks, corresponding to positional and rotational elements with respect to $\rm\mathbf{ICR} =\begin{bmatrix} X_v, Y_l, Y_r \end{bmatrix}^\top$ and scale factor $\begin{bmatrix} \alpha_l & \alpha_r \end{bmatrix}^\top$,
and $\mathbf{G}_{i-1}$ is the noise Jacobian matrix.
%
%
%
Due to space limitations, the detailed derivations of these matrices can be found in our companion technical report~\cite{tr_icr}.


Additionally, 
as the skid-steer robot navigates on ground surface, its positions within a short period of time should be well modeled by a quadratic polynomial~\cite{zhang2019large}:
\begin{align}~\label{eq:manifold}
	M_p({}^{\mathbf G}\mathbf{p}_{\mathbf O}) &=  \big{|} \big{|} \frac{1}{2} \begin{bmatrix}
	 {}^{\mathbf G} p_{\mathbf{O}x} \\  {}^{\mathbf G} p_{\mathbf{O}y}
	\end{bmatrix}^\top \mathbf{A} \begin{bmatrix}
	{}^{\mathbf G} p_{\mathbf{O}x} \\  {}^{\mathbf G} p_{\mathbf{O}y}
	\end{bmatrix} + \mathbf{B}^\top \begin{bmatrix}
	{}^{\mathbf G} p_{\mathbf{O}x} \\  {}^{\mathbf G} p_{\mathbf{O}y}
	\end{bmatrix} + {}^{\mathbf G}\mathbf{p}_{\mathbf{O}z} + c \big{|} \big{|}_{\bm \Lambda_{mp}} \\
	{\rm with}~~ \mathbf{A} &= \begin{bmatrix}
	a_1 & a_2\\ a_2 & a_3
	\end{bmatrix} ,~~~ \mathbf{B} = \begin{bmatrix}
	b_1 \\ b_2
	\end{bmatrix}
\end{align}
where $\mathbf{m} = \left[  a_1, a_2, a_3, b_1, b_2, c \right]^\top$ are the manifold parameters. 
%
Note also that the roll and pitch of the ground robot should be consistent with the normal of the motion manifold (ground surface), 
which can be expressed as follows:
 \begin{align}
 	M_r({}^{\mathbf G}_\mathbf{O}\mathbf{R}, {}^{\mathbf G} \mathbf{p}_\mathbf{O} ) = \big{|} \big{|} \lfloor {}^{\mathbf G}_\mathbf{O}\mathbf{R} \mathbf{e}_3  \rfloor_{12}* \frac{\partial M_p }{\partial {}^{\mathbf G} \mathbf{p}_\mathbf{O} } \big{|} \big{|}_{\bm \Lambda_{mr}}^2
 \end{align}
where $\lfloor \mathbf{v} \rfloor_{12}$ denotes the first and second rows of the symmetric matrix of the 3D vector $\mathbf{v}$. 
At this point, the motion manifold constraint for  all the poses in the current sliding window $i \in \lbrace k-s, \dots, k-1, k\rbrace$ can be written as:
\begin{align}
\mathcal{C}_{manifold}({}^{\mathbf G}_{\mathbf{O}_i}{ \mathbf{R}}, {}^{\mathbf G}{\mathbf{p}}_{\mathbf{O}_i},  \mathbf{m} ) = M_p({}^{\mathbf G}\mathbf{p}_{\mathbf{O}_i}) + M_r({}^{\mathbf G}_{\mathbf{O}_i}\mathbf{R}, {}^{\mathbf G} \mathbf{p}_{\mathbf{O}_i} )
\end{align}

\newcommand{\bO}{\mathbf{O}}
\newcommand{\bL}{\mathbf{L}}
\newcommand{\bT}{\mathbf{T}}
\newcommand{\bp}{\mathbf{p}}
\newcommand{\bl}{\mathbf{l}}
\newcommand{\ba}{\mathbf{a}}
\newcommand{\bomega}{\boldsymbol{\omega}}
\newcommand{\bo}{\mathbf{o}}
\newcommand{\bc}{\mathbf{c}}
\newcommand{\bz}{\mathbf{z}}
\newcommand{\bt}{\mathbf{t}}
\newcommand{\bQ}{\mathbf{Q}}
\newcommand{\bC}{\mathbf{C}}
\newcommand{\bA}{\mathbf{A}}
\newcommand{\bB}{\mathbf{B}}
\newcommand{\bG}{\mathbf{G}}
\newcommand{\bn}{\mathbf{n}}
\newcommand{\bI}{\mathbf{I}}
\newcommand{\bR}{\mathbf{R}}
\newcommand{\bD}{\mathbf{D}}
\newcommand{\bv}{\mathbf{v}}
\newcommand{\bJ}{\mathbf{J}}
\newcommand{\be}{\mathbf{e}}
\newcommand{\btheta}{\boldsymbol{\theta}}
\newcommand{\bepsilon}{\boldsymbol{\epsilon}}
\newcommand{\bgamma}{\boldsymbol{\gamma}}
\newcommand{\boldeta}{\boldsymbol{\eta}}

\section{Observability Analysis}

An important prerequisite condition for the proposed localization algorithm to work properly is that the skid-steering kinematic parameter vector, $\boldsymbol \xi$, is locally observable (or identifiable\footnote{Since derivative of $\boldsymbol \xi$ is modeled by zero-mean Gaussian, we here use observability and identifiability interchangeably.})~\cite{Bar-Shalom1988}. Therefore, in this section, we provide detailed observability analysis. We note that, it is also interesting to investigate the observability properties by applying the proposed method with monocular camera and odometer only (without having IMU). This will examine whether skid-steering robots can be localized with reduced number of sensors, and emphasize the importance of our choice of adding the IMU.

%
\subsection{
Observability of $\boldsymbol \xi$ with a monocular camera and odometer}
\label{sec:mono_ICR_5}
To conduct our analysis, we follow the idea of~\cite{li2014online}, in which information provided by each sensor is firstly investigated and subsequently combined together for deriving the final results. By doing this, 
`abstract' measurements instead of the `raw' measurements are used for analysis, which greatly simplifies our derivation.
A moving monocular camera is able to 
provide information on rotation and up-to-scale position with respect to the initial camera frame~\cite{hartley2003multiple}~\cite{li2014online}. Equivalently, we can say that a moving camera is able to provide the following two types of measurements: (\romannumeral1) camera's angular velocity and (\romannumeral2) its up-to-scale linear velocity:
\begin{subequations}
    \label{eq:camera_measurement}
    \begin{align}
    {}^{\mathbf{C}_{(t)}}{\bomega}_m = {}^{\mathbf{C}_{(t)}}\bomega + \mathbf{n}_{\omega} (t)\\
    {}^{\mathbf{C}_{(t)}} {\bv}_m = s^{-1}\cdot{}^{\mathbf{C}_{(t)}}\bv + \mathbf{n}_v (t)
    \end{align}
\end{subequations}
where $\mathbf{n}_{\omega} (t)$ and $\mathbf{n}_{v} (t)$ are the white noises, and ${}^{\mathbf{C}_{(t)}}{\bomega}$ and ${}^{\mathbf{C}_{(t)}}{\bv}$ are true angular and linear velocities of camera with respect to global frame expressed in camera frame respectively. Finally, $s$ is an unknown scale factor.
We also note that, since the camera to odometer extrinsic parameters are precisely known in advance, the camera measurements can be further denoted as:
\begin{align}
    \label{eq:camera_measurement1}
\breve{{\bomega}}(t) = {}^{\mathbf O} _{\mathbf C} \mathbf R {}^{\mathbf{C}_{(t)}}{\bomega}_m,\,\,
\breve{{\bv}}(t) = {}^{\mathbf O} _{\mathbf C} \mathbf R {}^{\mathbf{C}_{(t)}} {\bv}_m
\end{align}
We will later show that this will simplify the analysis.

On the other hand, as mentioned in Sec.~\ref{sec:icr-kinmatics}, odometer provides observations for the speed of left and right wheels, i.e., $o_{l}$ and $o_{r}$ respectively. By linking $o_{l}$,  $o_{r}$, $\breve{{\bomega}}(t) $, $\breve{{\bv}}(t)$, and kinematic parameter vector $\boldsymbol \xi$ together, the observability properties can be analyzed in details. 
We also note that, during the observability analysis, the zero-mean noise terms are ignored, since they will not change our conclusions.

By ignoring the noise terms, the following equation holds:
\begin{align}
	\label{eq:linear_velocity_relation}
	^{\mathbf{O}}\bv  &= 	-\lfloor{ { \breve\bomega}} \rfloor
	^{\mathbf O} \bp_{\mathbf C}+  s \cdot \breve{\bv} 
\end{align}
where $^{\mathbf O} \bp_{\mathbf C}$ is the known extrinsic parameter and $^{\mathbf{O}}\bv$ is velocity of the odometer frame with respect to the global frame expressed in the odometer frame.
Substituting Eq.~\ref{eq:camera_measurement1} and~\ref{eq:linear_velocity_relation} into 
Eq.~\ref{eq:ICRmat_5par} leads to :
\begin{align}
\begin{bmatrix}
\begin{bmatrix}
\breve{\omega} {}^{\mathbf O} y_{\mathbf C} \\
-\breve{\omega} {}^{\mathbf O} x_{\mathbf C} 
\end{bmatrix} + 
s \begin{bmatrix}
\breve{v} _{x} \\
\breve{v} _{y}
\end{bmatrix} \\ 
\breve{\omega}
\end{bmatrix} &= \frac{1}{\Delta Y} 
	\begin{bmatrix}
	- Y_r &  Y_l\\ 
	X_v  & - X_v\\
	-1 & 1
	\end{bmatrix}
	\begin{bmatrix}
	\alpha_l &0 \\ 0 &\alpha_r
	\end{bmatrix}
	\begin{bmatrix}
	o_l \\ o_r
	\end{bmatrix} \notag \\
	&= \begin{bmatrix}
	\breve{\omega} Y_l \\
	- \breve{\omega} X_v \\
	\frac{1}{\Delta Y}
	\begin{bmatrix}
	-1 & 1
	\end{bmatrix}
	\begin{bmatrix}
	\alpha_l & 0 \\
	0 & \alpha_r
	\end{bmatrix}
	\begin{bmatrix}
	o_l \\ o_r 
	\end{bmatrix} 
	\end{bmatrix}
	+
	\begin{bmatrix}
	\alpha_l o_l \\ 0 \\ 0
	\end{bmatrix}
\end{align} 
where ${}^{\mathbf O} x_{\mathbf C}, {}^{\mathbf O} y_{\mathbf C}$ are the first and second element of ${}^{\mathbf O} \bp _{\mathbf C}$, and 
$\breve{v} _{x} , \breve{v} _{y} $ are the first and second element of $\breve{\bv}$. For brevity, we use $\breve{\omega}$ to denote the third element of $\breve{ \bomega}$.
By defining $\beta_r = \Delta Y^{-1} \alpha _r$, and
$\beta_l = \Delta Y^{-1} \alpha _l$,  we can write
\begin{align}
\label{eq:constraints_mono_5}
	\begin{bmatrix}
	\begin{bmatrix}
	\breve{\omega} {}^{\mathbf O} y_{\mathbf C} \\
	-\breve{\omega} {}^{\mathbf O} x_{\mathbf C} 
	\end{bmatrix} + 
	s \begin{bmatrix}
	\breve{v} _{x} \\
	\breve{v} _{y}
	\end{bmatrix} \\ 
	\breve{\omega}
	\end{bmatrix} = \begin{bmatrix}
	\breve{\omega} Y_l \\
	- \breve{\omega} X_v \\
	- \beta_l o_l + \beta_r o_r
	\end{bmatrix}
	+
	\begin{bmatrix}
	\beta_l \Delta Y o_l \\ 0 \\ 0
	\end{bmatrix}
\end{align}
Note that, this equation only contains 1) sensor measurements, and 2) a combination of vision scale factors and skid-steering kinematics:
\begin{align}
\bepsilon = \begin{bmatrix}
X_v& Y_l& Y_r& \alpha_l& \alpha_r& s
\end{bmatrix}^\top \notag
\end{align}

The identifiability of $\epsilon$ can be described as follows:
\begin{lem}
	By using measurements from a monocular camera and wheel odometers, $\bepsilon$ is not locally identifiable.
\end{lem}
\begin{proof}
	 $\bepsilon$ is locally identifiable if and only if $\bar{\bepsilon}$ is locally identifiable:
	\begin{align}
	\label{eq:simple}
	\bar{\bepsilon} = \begin{bmatrix} Y_l& \Delta Y&  X_v& \beta_l& \beta_r& s \end{bmatrix}^\top \notag
	\end{align} 
	
	By expanding Eq.~\ref{eq:constraints_mono_5}, we can write the following constraints:
	\begin{subequations}
		\label{eq:constraints_mono_5_ex}
		\begin{align}
		c_x(\bar{\bepsilon}, t) &= \breve{\omega}(t) {}^{\mathbf O} y_{\mathbf C} + s \breve{v} _{x}(t) - \breve{\omega}(t) Y_l - \beta_l \Delta Y o_l(t) = 0\\
		c_y(\bar{\bepsilon}, t) &=  -\breve{\omega}(t) {}^{\mathbf O} x_{\mathbf C} + s \breve{v} _{y}(t) + \breve{\omega}(t) X_v = 0\\
		c_{\omega}(\bar{\bepsilon}, t) &= \breve{\omega}(t) + \beta_l o_l(t) - \beta_r o_r(t) =0
		\end{align}
	\end{subequations}
	A necessary and sufficient condition of $\bar{\bepsilon}$ to be locally identifiable is the following observability matrix has full column rank~\cite{van2009identifiability}:
		\begin{align}
		\label{eq:occ}
	\mathcal{O}_c = \begin{bmatrix}
	\mathbf{D}(t_0)^\top & \mathbf{D}(t_1)^\top& \dots& \mathbf{D}(t_s)^\top
	\end{bmatrix}^\top
	\end{align}
	where
		\begin{align}
		\mathbf{D}(t) &= \begin{bmatrix}
	\frac{\partial c_x(\bar{\bepsilon}, t)}{\partial \bar{\bepsilon}} & \frac{\partial c_y(\bar{\bepsilon}, t)}{\partial \bar{\bepsilon}} & \frac{\partial c_{\omega}(\bar{\bepsilon}, t)}{\partial \bar{\bepsilon}}
	\end{bmatrix} ^\top \notag \\
	\label{eq:d}
	&= \begin{bmatrix}\begin{smallmatrix}
	- \breve{\omega}(t)& -\beta_l o_l(t)& 0& -\Delta Y o_l(t)& 0& \breve{v} _{x}(t)\\
	0& 0& \breve{\omega}(t)& 0& 0& \breve{v} _{y}(t)\\
	0& 0& 0& o_l(t)& -o_r(t)& 0
	\end{smallmatrix}
	\end{bmatrix} 
	\end{align}
	%
	Putting Eq.~\ref{eq:d} back into Eq.~\ref{eq:occ} leads to:
	\begin{align}
	\label{eq:obser_matrix_5icr}
	\mathcal{O}_c = \begin{bmatrix}\begin{smallmatrix}
	- \breve{\omega}(t_0)& -\beta_l o_l(t_0)& 0& -\Delta Y o_l(t_0)& 0& \breve{v} _{x}(t_0)\\
	0& 0& \breve{\omega}(t_0)& 0& 0& \breve{v} _{y}(t_0)\\
	0& 0& 0& o_l(t_0)& -o_r(t_0)& 0 \\
	\vdots& \vdots& \vdots& \vdots& \vdots& \vdots\\
	- \breve{\omega}(t_s)& -\beta_l o_l(t_s)& 0& -\Delta Y o_l(t_s)& 0& \breve{v} _{x}(t_s)\\
	0& 0& \breve{\omega}(t_s)& 0& 0& \breve{v} _{y}(t_s)\\
	0& 0& 0& o_l(t_s)& -o_r(t_s)& 0
	\end{smallmatrix}
	\end{bmatrix}
	\end{align}
	By defining  $\mathcal{O}_c(:,1)$ the $i$th block columns of $\mathcal{O}_c$, the following equation holds:
	\begin{align}
	 (-{}^{\mathbf O} y_{\mathbf C} \!+\! Y_l)\cdot\mathcal{O}_c(:,1) \!+\! \Delta Y \cdot \mathcal{O}_c(:,2)  \!+\! (X_v - {}^{\mathbf O} x_{\mathbf{C}}) \cdot \mathcal{O}_c(:,3)\!+\! s \cdot \mathcal{O}_c(:,6)= \mathbf 0 \notag
	\end{align}
	which demonstrates that $\mathcal{O}_c$ is {\em not} of full column rank. This completes the proof.
\end{proof}

\subsection{Observability of $\boldsymbol \xi$ with a monocular camera, an IMU, and odometer}
When an IMU is added, the `abstract' measurement of visual-inertial estimation can be also derived.
Visual-inertial estimation provides: camera's local (\romannumeral1) angular velocity and (\romannumeral2) linear velocity, 
similar to vision only case (Eq.~\ref{eq:camera_measurement1}) without having scale effect~\cite{li2014online}.
Similarly to Eq.~\ref{eq:simple}, to simplify the analysis, we prove identifiability of $\bar{\boldsymbol \xi}$
instead of ${\boldsymbol \xi}$:
\begin{align}
\bar{\boldsymbol \xi} = \begin{bmatrix} Y_l& \Delta Y&  X_v& \beta_l& \beta_r\end{bmatrix}^\top \notag
\end{align}

\begin{lem}
	By using measurements from a monocular camera, an IMU, and wheel odometer, $\bar{\boldsymbol \xi}$ is locally identifiable, except for following degenerate cases: (\romannumeral1) velocity of one of the wheels, $o_l(t)$ or $o_r(t)$, keeps zero; (\romannumeral2) $\breve{\omega}(t)$ keeps zero; (\romannumeral3) $o_r (t)$, $o_l (t)$, and $\breve{\omega}(t)$ are all constants; (\romannumeral4) $o _{l}(t)$ is always proportional to $o_r(t)$.
\end{lem}

\begin{proof}
	Similarly to Eq.~\ref{eq:constraints_mono_5_ex}, by removing the scale factor, the constraints become:
	\begin{subequations}
		\label{eq:constraints_stereo_5_ex}
		\begin{align}
		c_x(\bar{\boldsymbol \xi}, t) &= \breve{\omega}(t) {}^{\mathbf O} y_{\mathbf C} +  \breve{v} _{x}(t) - \breve{\omega}(t) Y_l - \beta_l \Delta Y o_l(t) = 0\\
		c_y(\bar{\boldsymbol \xi}, t) &=  -\breve{\omega}(t) {}^{\mathbf O} x_{\mathbf C} +  \breve{v} _{y}(t) + \breve{\omega}(t) X_v = 0\\
		c_{\omega}(\bar{\boldsymbol \xi}, t) &= \breve{\omega}(t) + \beta_l o_l(t) - \beta_r o_r(t) =0
		\end{align}
	\end{subequations}
	The observability matrix for $\bar{\boldsymbol \xi}$ then becomes:
	\begin{align}
	\label{eq:obser_matrix_5icr_stereo}
	\mathcal{O}_c = \begin{bmatrix}\begin{smallmatrix}
	- \breve{\omega}(t_0)& -\beta_l o_l(t_0)& 0& -\Delta Y o_l(t_0)& 0\\
	0& 0& \breve{\omega}(t_0)& 0& 0\\
	0& 0& 0& o_l(t_0)& -o_r(t_0) \\
	\vdots& \vdots& \vdots& \vdots& \vdots\\
	- \breve{\omega}(t_s)& -\beta_l o_l(t_s)& 0& -\Delta Y o_l(t_s)& 0\\
	0& 0& \breve{\omega}(t_s)& 0& 0\\
	0& 0& 0& o_l(t_s)& -o_r(t_s)
	\end{smallmatrix}
	\end{bmatrix}
	\end{align}
	which can be simplified by linear operations:
	\begin{align}
	\label{eq:obser_matrix_5icr_stereo_af}
	\mathcal{O}_c = \begin{bmatrix}\begin{smallmatrix}
	- \breve{\omega}(t_0)& o_l(t_0)& 0& 0& 0\\
	0& 0& \breve{\omega}(t_0)& 0& 0\\
	0& 0& 0& o_l(t_0)& -o_r(t_0) \\
	\vdots& \vdots& \vdots& \vdots& \vdots\\
	- \breve{\omega}(t_s)& o_l(t_s)& 0& 0& 0\\
	0& 0& \breve{\omega}(t_s)& 0& 0\\
	0& 0& 0& o_l(t_s)& -o_r(t_s) 
	\end{smallmatrix}		
	\end{bmatrix}
	\end{align}
	There are four special cases to make $\mathcal{O}_c$ not of full column rank:
	(\romannumeral1) velocity of one of the wheels, $o_l(t)$ or $o_r(t)$, keeps zero; (\romannumeral2) $\breve{\omega}(t)$ keeps zero; (\romannumeral3) $o_r (t)$, $o_l (t)$, and $\breve{\omega}(t)$ are all constants; (\romannumeral4) $o _{l}(t)$ is always proportional to $o_r(t)$. If none of those conditions is met, $\mathcal{O}_c$ is of full column rank. 
	This reveals that, under general motion, $\boldsymbol \xi$ is locally identifiable. This completes the proof.
\end{proof}


\begin{figure}[t]
	\centering 
	\subfigure[]{ 
		\includegraphics[width=0.8in]{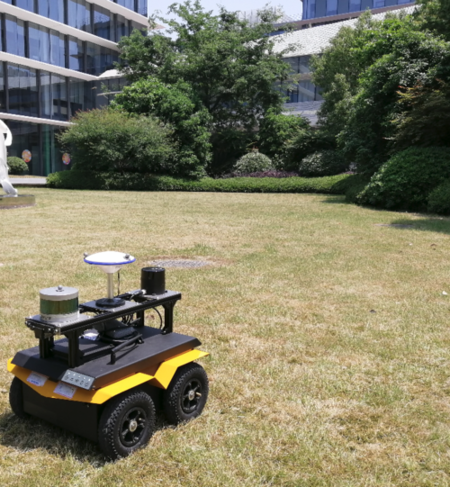} 
	} 
	\subfigure[]{ 
		\includegraphics[width=0.8in]{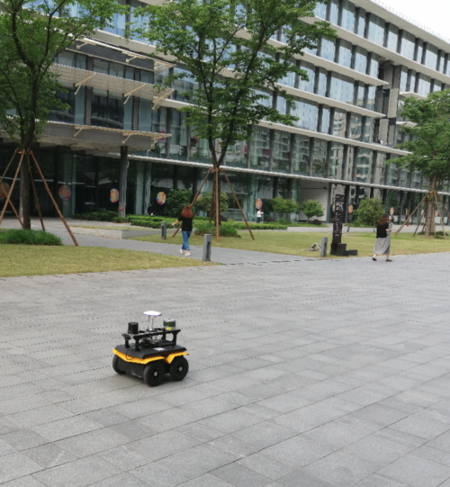} 
	} 
	\subfigure[]{ 
		\includegraphics[width=0.8in]{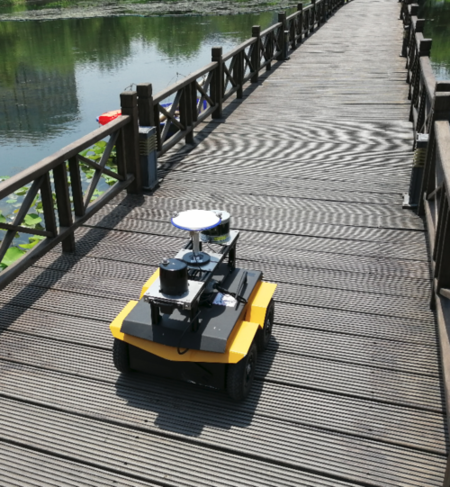} 
	} 
	\subfigure[]{ 
		\includegraphics[width=0.8in]{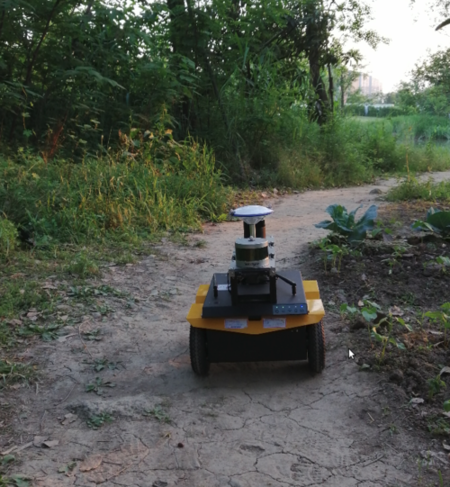} 
	} \\
	\subfigure[]{ 
		\includegraphics[width=0.8in]{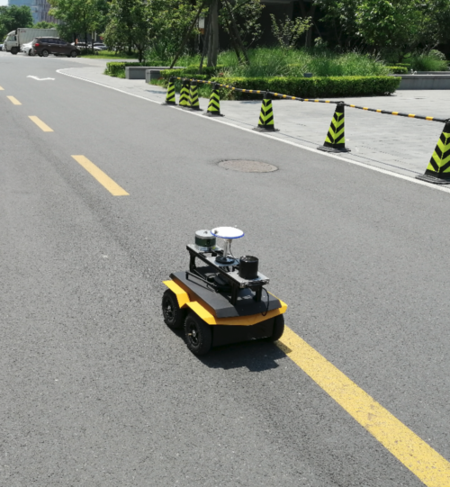} 
	} 
	\subfigure[]{ 
		\includegraphics[width=0.8in]{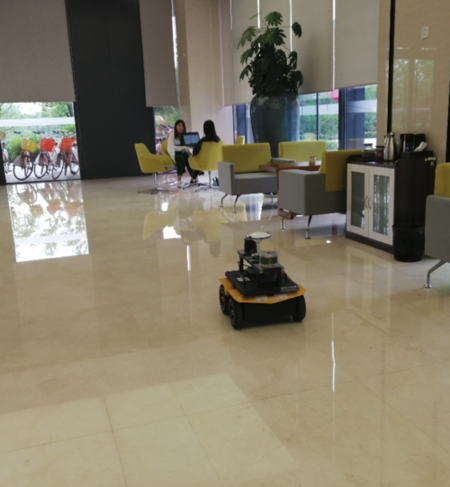} 
	} 
	\subfigure[]{ 
		\includegraphics[width=0.8in]{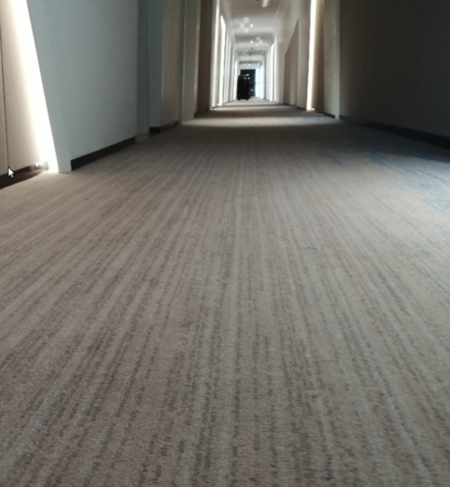} 
	} 
	\subfigure[]{ 
		\includegraphics[width=0.8in]{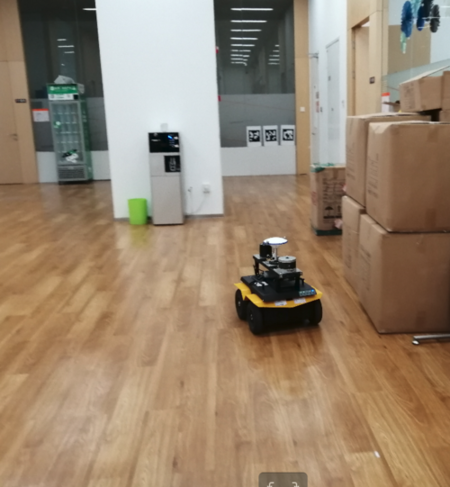} 
	} 	
	\caption{Skid-Steering robot traverses variable terrains: (a) lawn, (b) cement brick, (c) wooden bridge, (d)  muddy road, (e) asphalt road, (f) ceramic tiles, (g) carpet, and (h) wooden floor.} \label{fig:terrain}
\end{figure}

\section{Experimental Results}
\label{sec:exp}

%

\begin{figure}[H]
	\centering 
	\includegraphics[scale=.24]{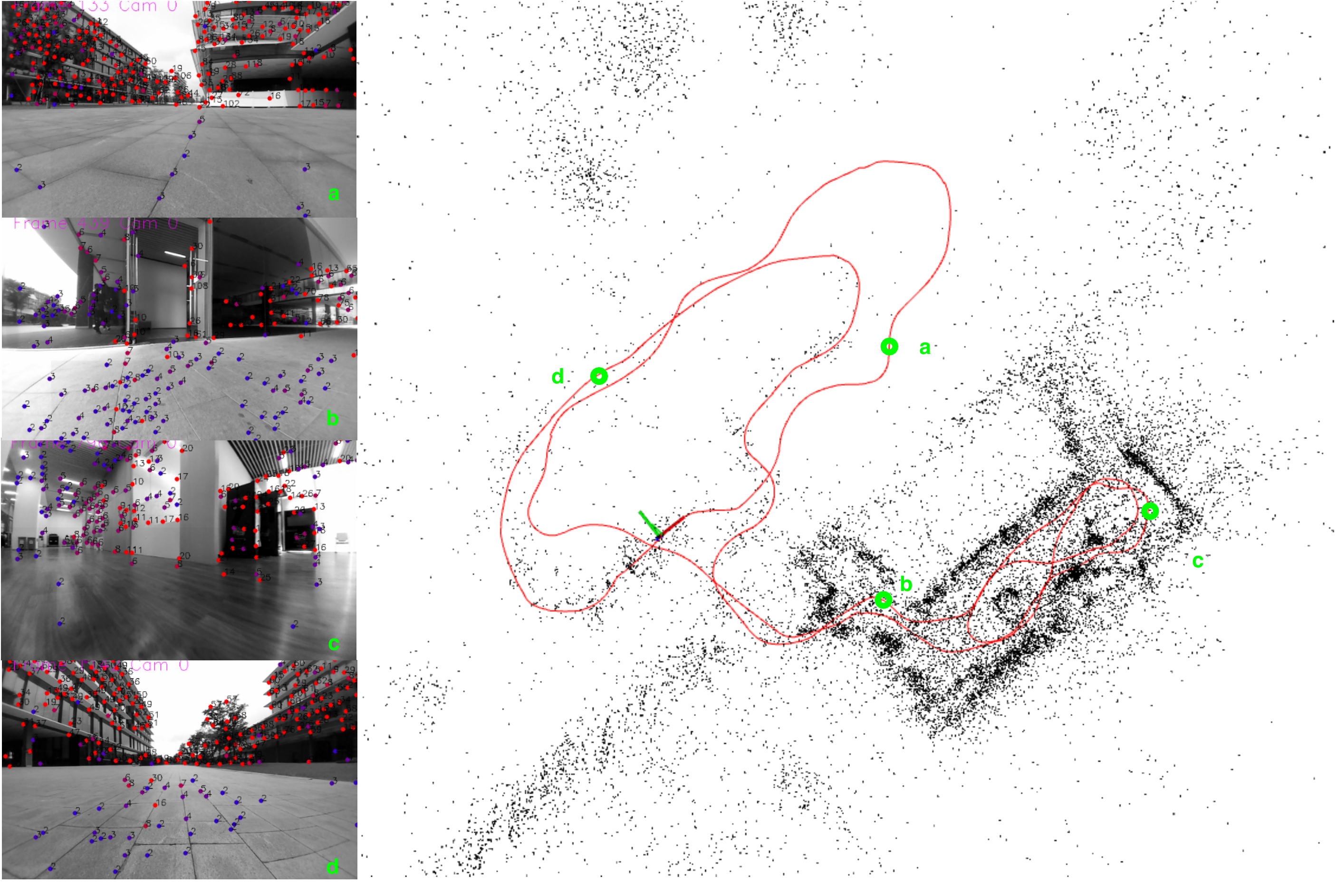} 
	\caption{Skid-steering robots traversed outdoors and indoors. The left part shows the representative images with visual features recorded at positions marked by green circles respectively. The right part shows the estimated trajectory red curve, and estimated 3D landmarks by black dots.} \label{fig:visual_map}
\end{figure}
\begin{table*}[t]
	\renewcommand{\arraystretch}{1.5}
	\caption{Estimating $\boldsymbol \xi$ or not: Final drift test.}
	\label{tb:drift}
	\begin{center}
		\resizebox{0.95\textwidth}{!}
		{
			\begin{tabular}{c|c|c|c|c|c|c|c|c|c|c}\cline{4-11}
				\multicolumn{3}{c|}{}&\multicolumn{4}{c|}{\textbf{W/  $\boldsymbol \xi$}}&\multicolumn{4}{c}{\textbf{W/O $\boldsymbol \xi$}}\\\cline{1-11}
				\textbf{Sequence}&\textbf{ Length(m)}& \textbf{Terrain}&\textbf{Norm(m)}&\textbf{x(m)}&\textbf{y(m)}&\textbf{z(m)}&\textbf{Norm(m)}&\textbf{x(m)}&\textbf{y(m)}&\textbf{z(m)}\\\hline
				CP02-2019-04-03-16-24-35& 232.30& (b) &  4.3617  & 0.399317& 4.31398& 0.504615&6.2460 & 0.73324& 4.84088& 3.87821 \\
				CP01-2019-04-24-17-17-35& 193.63& (f) &  0.4497   & -0.0717224& 0.0935681& 0.434018&	3.7052& 2.77887& -1.15464& 2.1617 \\
				CP01-2019-04-19-15-42-40& 632.64& (b,f) &  1.9297 &0.793459& 1.27422& 1.21264 & 28.9106	& -27.5063& -6.27234& 6.31525 \\
				CP01-2019-04-19-15-56-09& 629.96& (b,f) & 5.9323  &-5.77259& -0.738212& 1.15102&	35.9580& -33.9281& -10.1135& 6.29139 \\
				CP01-2019-04-19-16-09-53& 626.83& (b,f) &  1.5361 &0.505865& 0.777903& 1.22411 &	31.0437& -29.3167& -7.9511& 6.4049 \\
				
				CP01-2019-04-25-11-16-06& 212.59& (g) &  8.5111   &5.89016& -6.12049& 0.532888 &	10.0050& 8.17044& -5.28882& 2.31773 \\
				CP01-2019-05-08-18-06-43& 51.44& (a) &  0.3207   &-0.211653& -0.232218& -0.064136&	 0.8351& -0.30158& -0.155709& 0.762975 \\
				CP01-2019-05-08-17-50-51& 204.81& (e) &  0.7255   &-0.218743& 0.0247195& 0.691284& 2.2180	& -0.439685& 0.0659946& 2.17303 \\
				CP01-2019-05-09-10-59-53& 77.63& (c) &   0.3402  &-0.00968671& 0.339343& -0.0217497&	1.0127& 0.245799& 0.239745& 0.95274 \\
				CP01-2019-05-09-11-09-51& 27.09& (a) &  0.2239   &-0.125968& -0.0555045& 0.176533 & 0.5191&  -0.136975& -0.154332& 0.476361 \\
				
				CP01-2019-05-09-11-24-39& 270.41& (e,b) & 0.6447  & 0.263564& -0.252776& 0.531292,&	3.1481& 0.261559& -0.326145& 3.1202 \\
				CP01-2019-05-08-17-42-01& 436.19& (e) &  0.7474   &0.184071& 0.1038& 0.716877&	7.0835& 0.24108& 5.48359& 4.47751 \\
				CP01-2019-05-08-18-13-21& 28.64& (d) &   0.0697   &0.0325226& -0.0173858& 0.0590985& 0.3496	& 0.0587353& 0.100583& 0.329623 \\
				CP01-2019-04-19-14-57-38& 372.15& (b) &   8.7967  & 8.65849& 1.16737& 1.02493&	13.2613 &12.7065& 1.60303& 3.44035 \\
				CP02-2019-04-25-20-49-04& 81.03& (h) &  2.2940   &-2.2311& 0.459289& -0.271489,& 2.5641	&-2.26811& 0.200971& 1.17896 \\
				
				CP02-2019-04-25-21-07-46& 53.49& (h) &   0.6084  &-0.515633& 0.30291& -0.111944& 1.0931	& -0.757592& 0.111233& 0.780112 \\
				CP01-2019-05-27-14-32-36& 110.55& (b) &   0.7769   & -0.26637& -0.723188& 0.098119&1.3461	&  -0.286492& -0.82653& 1.0231 \\
				CP01-2019-05-27-14-41-33& 104.63& (h) &  0.3790   &0.327562& -0.0804229& 0.172905& 1.3777	& 0.422314& -0.769413& 1.06189 \\
				CP01-2019-05-27-14-46-21& 214.66& (b,h) & 1.5367  & -1.09328& -0.927257& 0.553504&2.4919	& -0.319398& -1.39199& 2.04208 \\
			CP01-2019-05-27-14-50-49& 254.30& (b,h) & 0.8219  & 0.307967& -0.0732949& 0.75849&3.1792	&-0.615917& -2.03648& 2.36239 
				\\\hline
			\end{tabular}
		}
	\end{center}
\end{table*}
As shown in Fig.~\ref{fig:robots_icr}, 
our experiments were conducted by two skid-steering robots with both `localization' sensors and `ground-truth sensors' equipped.
For `localization' sensors, we used a $10$Hz monocular global shutter camera at resolution of $640 \times 400$,  a $200$Hz Bosch BMI160 IMU, and $100$Hz wheel odometers.
The `ground truth' sensor mainly relies on RTK-GPS, who reports $1$Hz data when the signal is reliable.
The accuracy of RTK-GPS is at centimeter level.
%
%
%
%

The first experiment is to demonstrate the improvement of localization accuracy by estimating $\boldsymbol \xi$ (Eq.~\ref{eq:ICRmat_5par}) online.
As shown in Fig~\ref{fig:terrain}, we conducted experiments under different environments, i.e., (a) lawn, (b) cement brick, (c) wooden bridge, (d) muddy road, (e) asphalt road, (f) ceramic tiles, (g) carpet, and (h) wooden floor.
Fig.~\ref{fig:visual_map} shows the trajectory and visual features estimated by the proposed method on sequence "CP01-2019-05-27-14-50-49", in which the robot traversed outdoors and indoors.
Since GPS signal is not available in all tests (e.g., indoor tests), we here used final drift as the first error metric. To make this possible, we started and terminated each experiment at the same position. 
Two algorithms were implemented in this test: 1) the proposed one by explicitly estimating $\boldsymbol \xi$, 
and 2) using Eq.~\ref{eq:ordinarymodel} without modeling $\boldsymbol \xi$\footnote{In fact, Eq.~\ref{eq:ordinarymodel} can be considered as one-parameter approximation of skid-steering kinematics, if $b$ is probabilistically estimated.}.

In Table.~\ref{tb:drift}, we show the final drift values on 20 representative sequences. 
%
Since we used the two robots, we used the notation ``CP01, CP02" to 
denote the names of the robots. %
%
%
Table.~\ref{tb:drift} clearly demonstrates that when 
skid-steering kinematic parameters are estimated online, the localization accuracy can be significantly improved. In fact, in almost half of the tests, the errors are reduced by approximately a order of magnitude.
This validates our claim that to use odometer measurements of skid-steering robots, the complicated mechanism must be explicitly modeled to avoid accuracy loss. 
%
%
\begin{figure}[t]
	\centering 
	\subfigure[]{ 
		\includegraphics[width=1.1in]{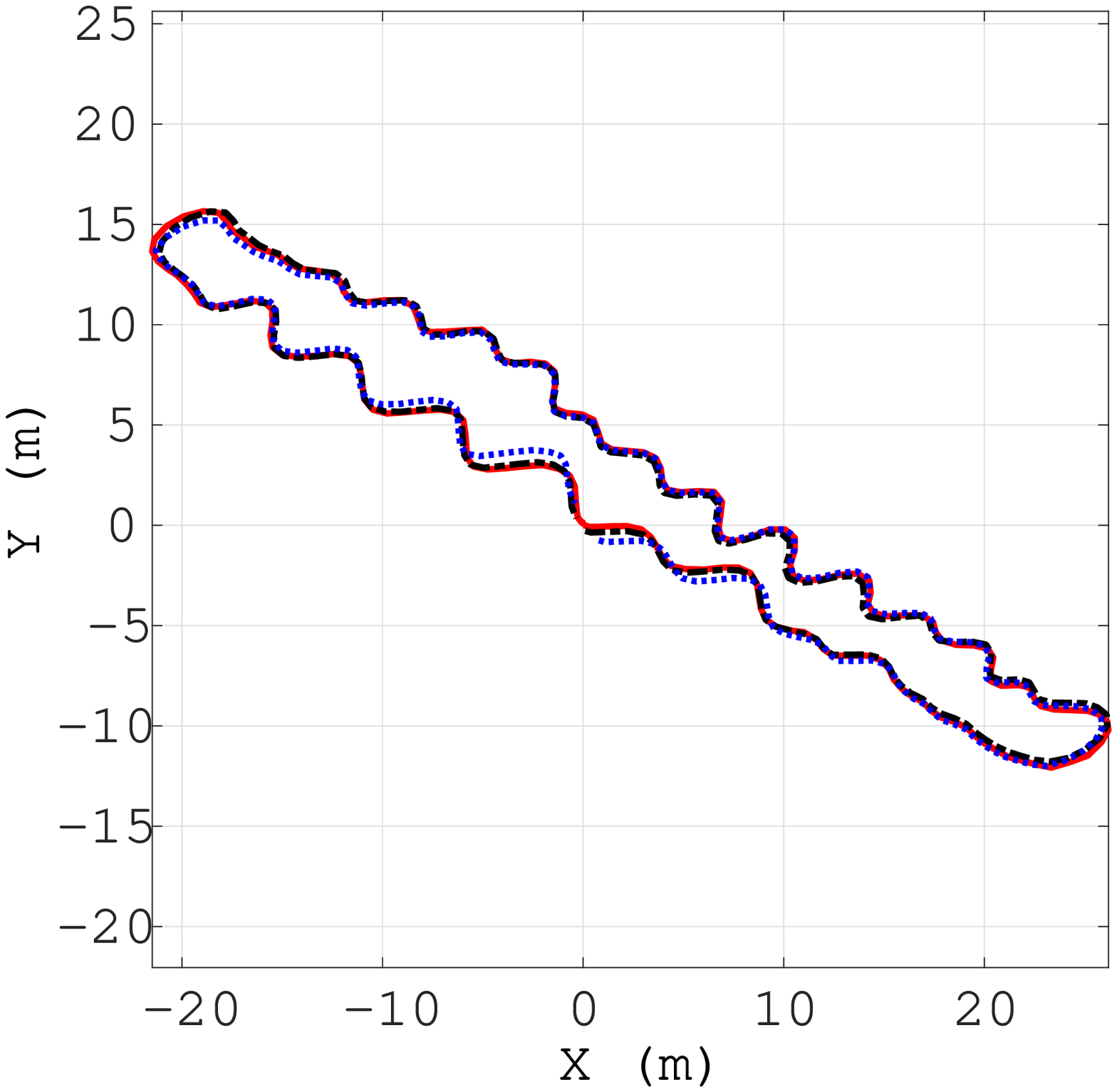} 
	} 
    \subfigure[]{ 
		\includegraphics[width=1.1in]{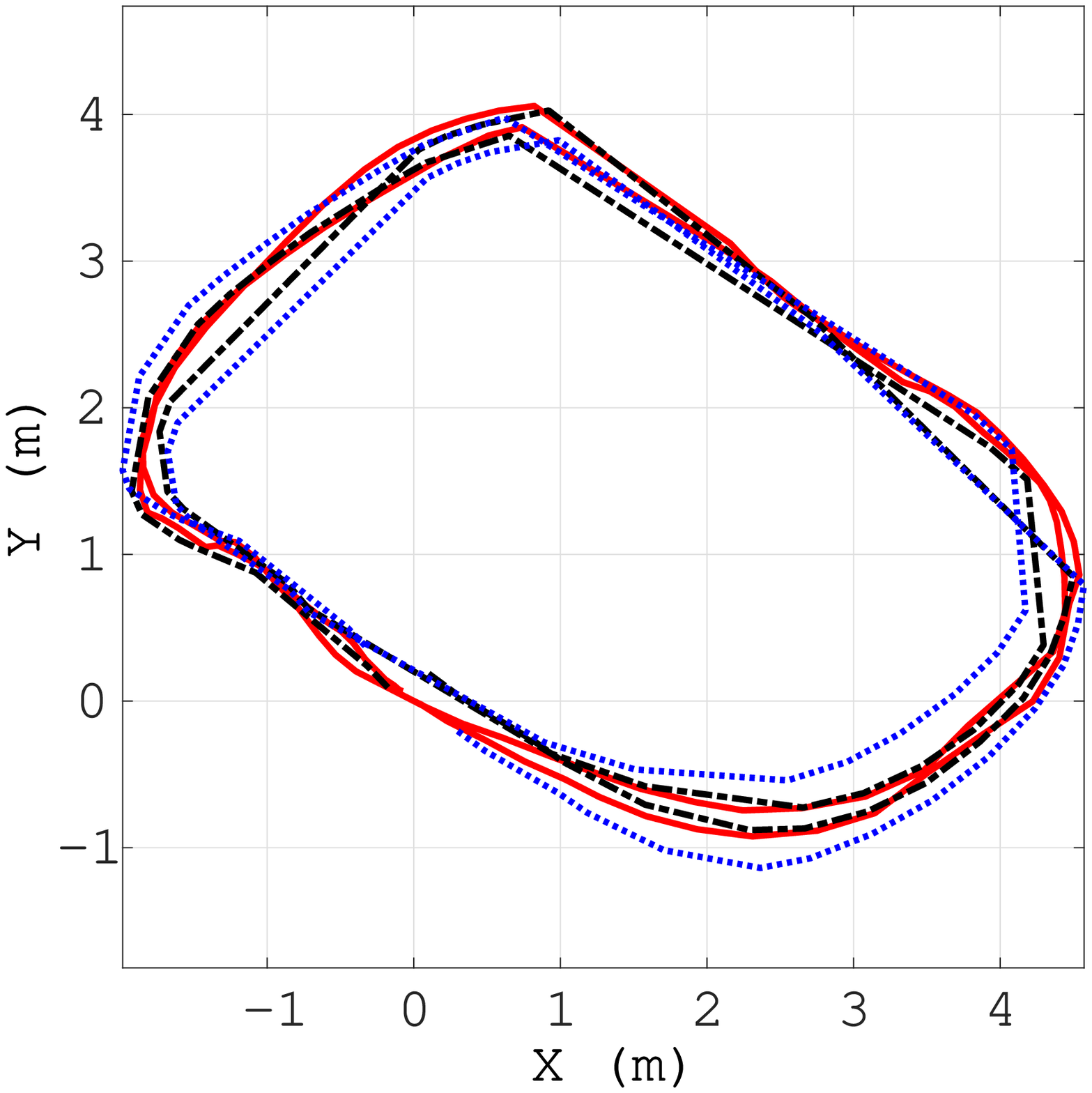} 
	}
    \subfigure[]{ 
		\includegraphics[width=1.1in]{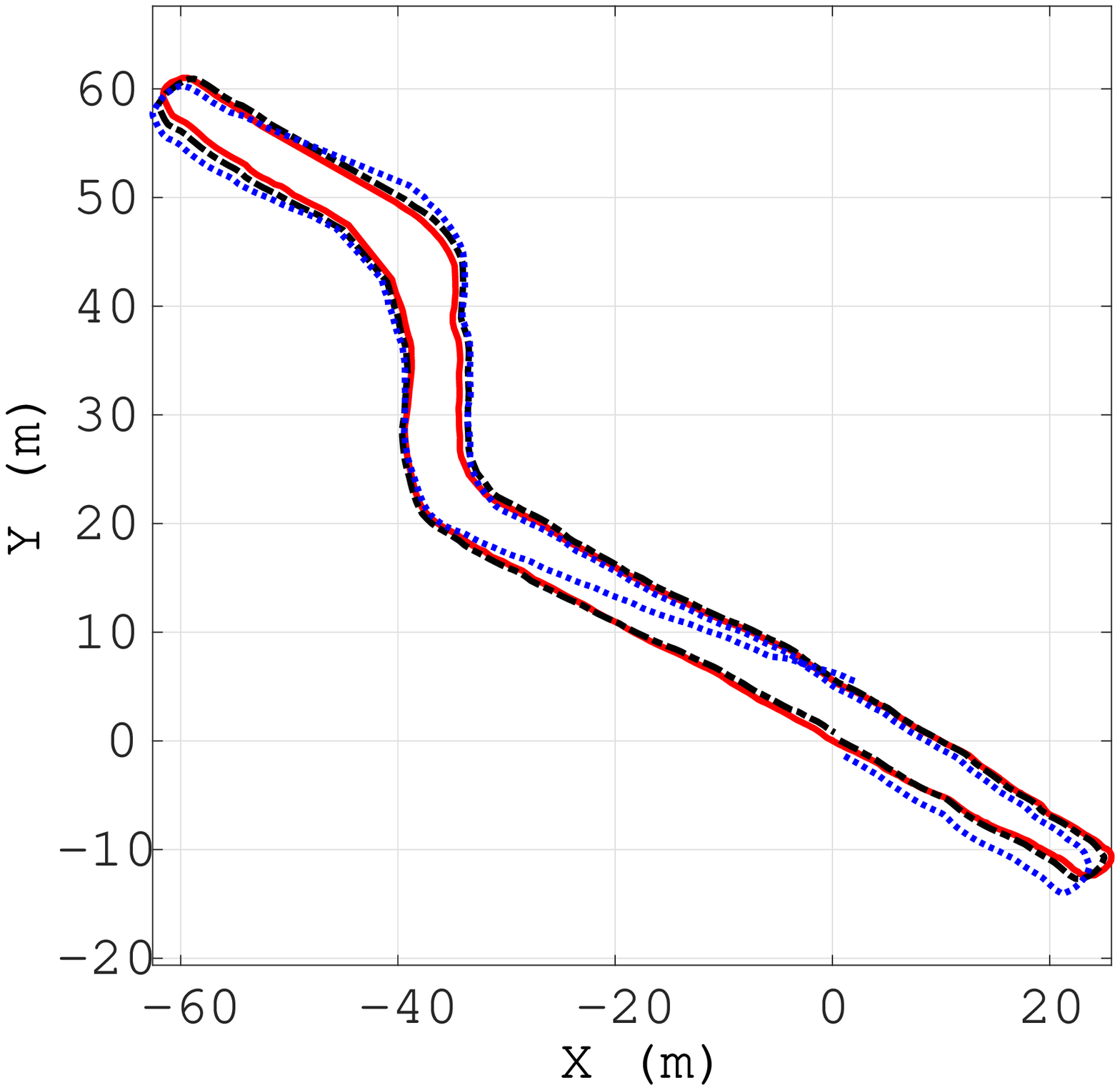} 
	} 
	\subfigure[]{ 
		\includegraphics[width=1.1in]{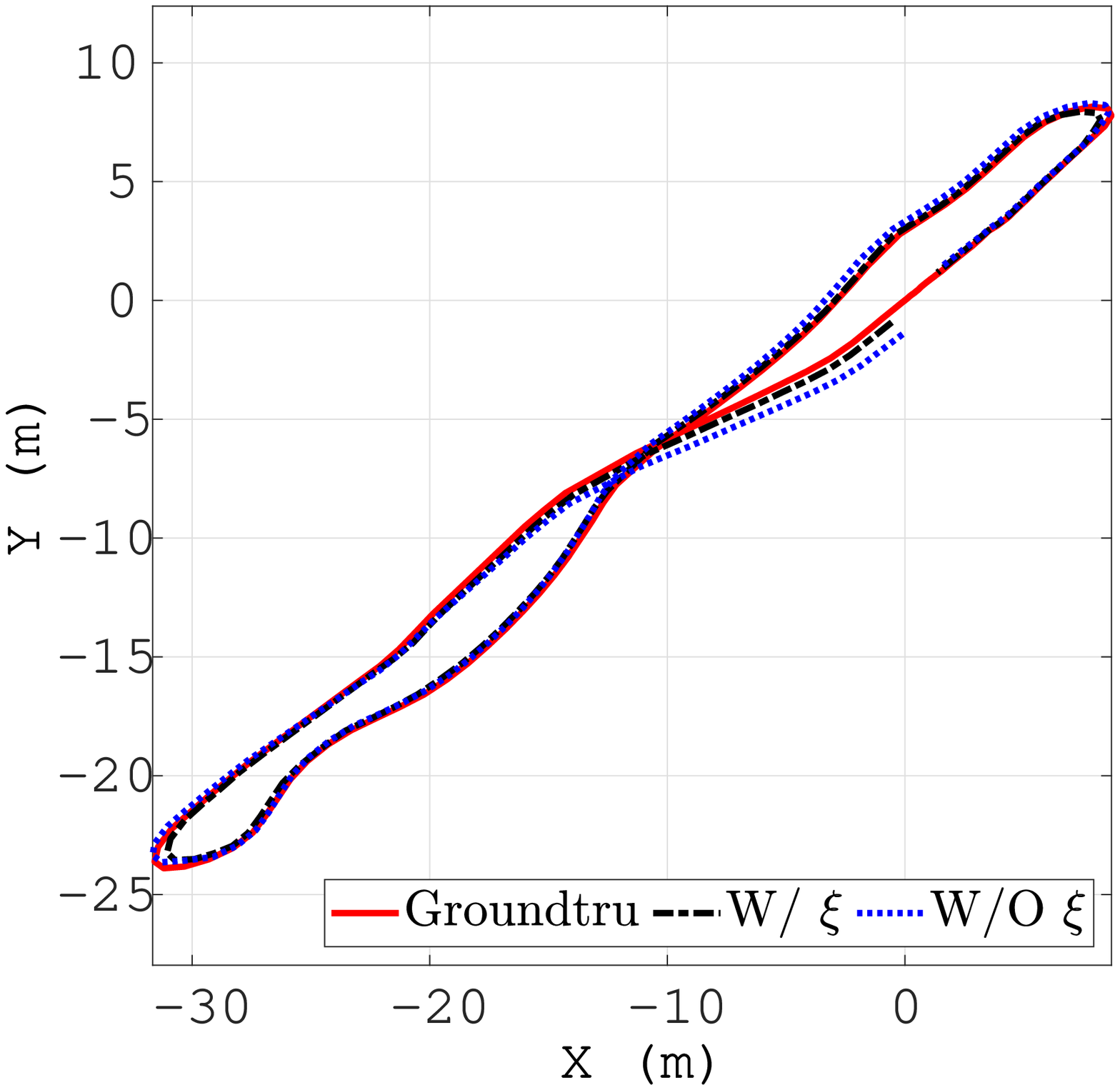} 
	} 	
	\caption{The estimated trajectory with and without $\boldsymbol \xi$. The trajectories of RTK-GPS (ground truth) are denoted by red firm lines, 
	the ones with modeling $\boldsymbol \xi$ by black dash dotted lines, and the ones without $\boldsymbol \xi$ by blue dotted lines.}
 ~\label{fig:traj_gps}
\end{figure}
In some sequences where GPS signals were available, we also evaluated the positional root mean square errors (RMSE)\cite{Bar-Shalom1988}. To compute that, we interpolated the estimated poses to get the ones corresponding to the timestamp of GPS measurements. The RMSE errors are shown in Table.~\ref{tb:gps_error}, which demonstrate that estimating $\boldsymbol \xi$ is beneficial for trajectory tracking. Trajectory estimates on representative sequences are shown in Fig.~\ref{fig:traj_gps}.
%
%
\begin{table*}[htb]
	\renewcommand{\arraystretch}{1.5}
	\caption{ Estimating $\boldsymbol \xi$ or not: RMSE Test}
	\label{tb:gps_error}
	\begin{center}
		\resizebox{0.6\textwidth}{!}
		{
			\begin{tabular}{c|c|c|c|c|c|c|c|c}\cline{4-9}
				\multicolumn{3}{c|}{}&\multicolumn{3}{c|}{\textbf{W/ $\boldsymbol \xi$}}&\multicolumn{3}{c}{\textbf{W/O $\boldsymbol \xi$}}\\\cline{1-9}
				\textbf{Sequence}&\textbf{ Length(m)}& \textbf{Terrain}&\textbf{Norm(m)}&\textbf{x(m)}&\textbf{y(m)}&\textbf{Norm(m)}&\textbf{x(m)}&\textbf{y(m)}\\\hline
					CP01-2019-04-19-15-42-40& 632.64& (b,f) &  0.82&  0.74&  0.37 &9.41&  6.14&  7.12\\
					CP01-2019-04-19-15-56-09& 629.96& (b,f) &  1.87&  1.06&  1.54& 9.94&  6.16&  7.80\\
					CP01-2019-04-19-16-09-53& 626.83& (b,f) &  1.56&  1.28&  0.89& 8.35&  5.70&  6.10\\

					CP01-2019-05-08-18-06-43& 51.44& (a) &  0.09&  0.07&  0.06& 0.22&  0.14&  0.17 \\
					CP01-2019-05-08-17-50-51& 204.81& (e) &  0.23&  0.17&  0.15& 0.44&  0.25&  0.36\\
					

					CP01-2019-05-08-17-42-01& 436.19& (e) &  0.57&  0.46&  0.33& 1.81&  0.99&  1.51\\
				    CP01-2019-04-19-14-57-38& 372.15& (b) &  1.62&  1.18&  1.12& 4.93&  3.69&  3.26 \\

					CP01-2019-05-27-14-32-36& 110.55& (b) &  0.37&  0.30&  0.22& 0.39&  0.16&  0.36

				\\\hline
			\end{tabular}
		}
	\end{center}
\end{table*}
\subsection{Convergence of Kinematic Parameters}
In this section, we show tests to demonstrate the convergence of $\boldsymbol{\xi}$ under general motion. Unlike the experiments in the previous section where relatively good initial values of kinematic parameters were used, we manually set `bad' initial value to kinematic parameters.
Specifically, we added following error terms to initial kinematic parameters used in the previous section (good values):
$\delta X_v = 0.08(m), \delta Y_l = 0.14(m), \delta Y_r = -0.1(m), \delta  \alpha_l =0.2, \delta  \alpha_r = 0.2 $.
We carried out tests on outdoor sequence  "CP01-2019-05-08-17-50-51" and indoor sequence "CP01-2019-05-27-14-41-33", which did not involve changes of terrain types on the fly.
In Fig.~\ref{fig:icr_converge}, the estimates of kinematic parameters are shown, along with the corresponding uncertainty envelopes.%
The results demonstrate that the kinematic parameters quickly converge to their correct values, and remains slow change rates for the rest of the trajectory.
The uncertainty envelopes also shrink quickly. 
The results exactly meet our theoretical expectations that $\boldsymbol \xi$ is locally identifiable under general motion.

\begin{figure}[H]
	\centering 
	\subfigure[]{ 
		\includegraphics[width=2.0in]{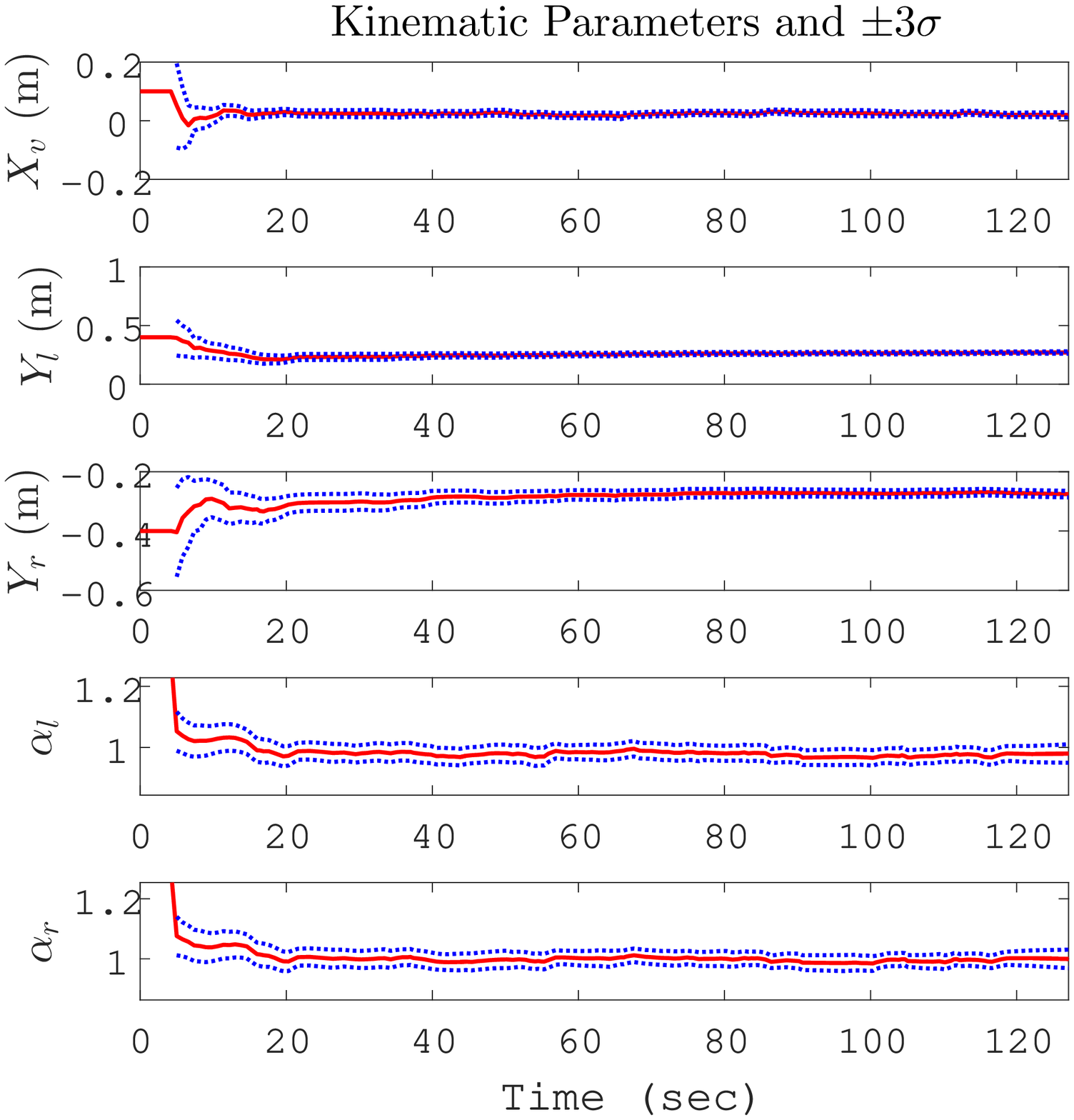} 
	} 
	\subfigure[]{ 
		\includegraphics[width=2.0in]{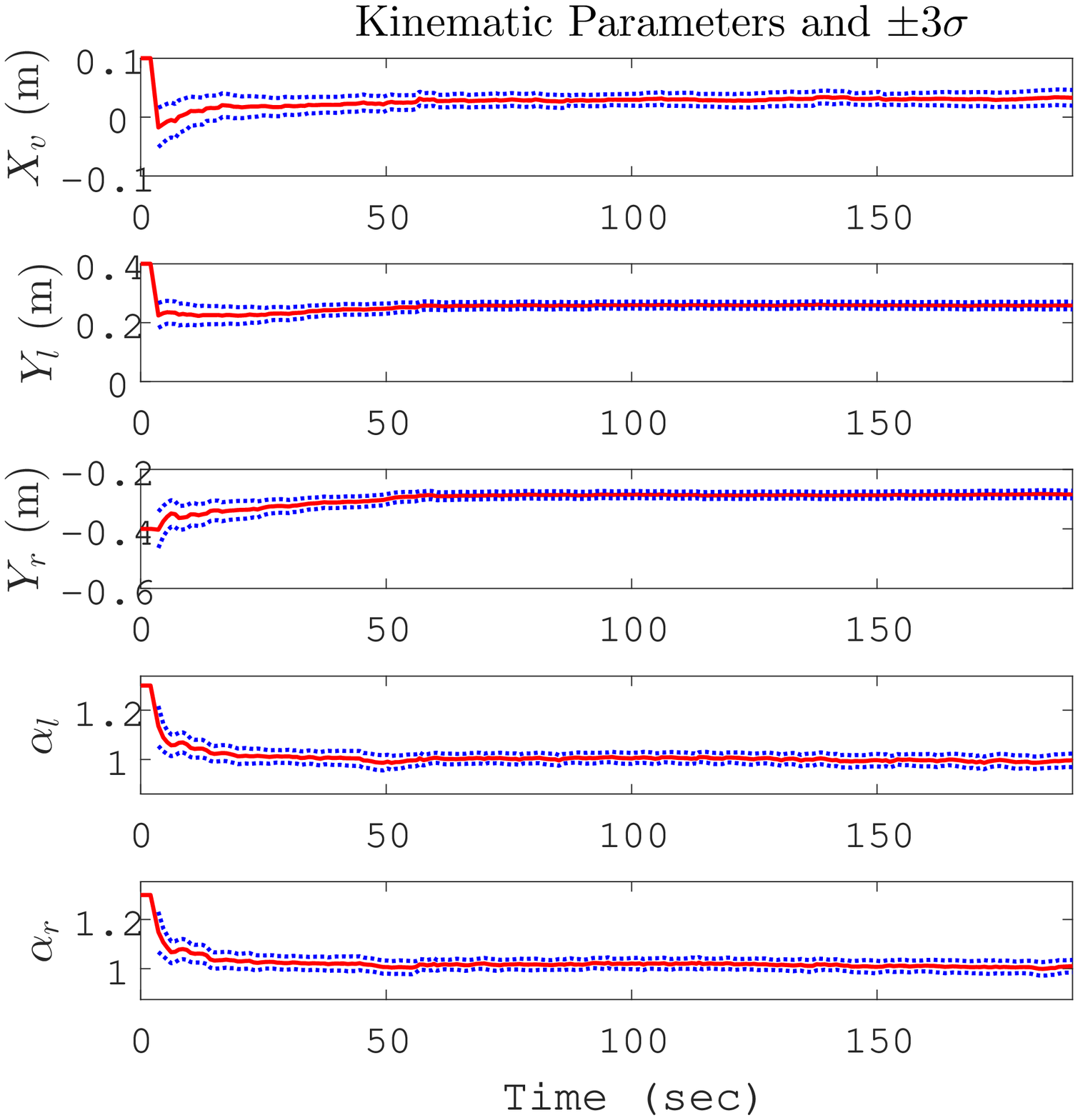} 
	} 	
	\caption{The estimated kinematic parameters and the associated $\pm 3 \sigma$ envelopes in two tests.}
 ~\label{fig:icr_converge}
\end{figure}
\section{Conclusions}

In this paper, we have developed a novel kinematics-constrained visual-inertial localization method specialized for skid-steering robots,
where a tightly-coupled sliding-window BA serves as the estimation engine for fusing multi-modal measurements.  
In particular, we have explicitly modeled the kinematics of skid-steering robots using both track ICRs and scale factors, 
in order to compensate for complex track-to-terrain interactions, imperfectness of mechanical design and terrain smoothness.
%
%
Moreover, we have carefully examined  the observability analysis, showing that the kinematic parameters are observable under general motion.
Extensive real-world validations confirm that online kinematic estimation significantly improves localization.


{
\def\bibfont{\footnotesize}
\printbibliography
}

%
%

\end{document}